\documentclass[runningheads]{llncs}

\usepackage{eccv}

\usepackage{eccvabbrv}

\usepackage{graphicx}
\usepackage{booktabs}
\usepackage{siunitx}
\usepackage{wrapfig}
\usepackage{caption}
\usepackage{subcaption}
\usepackage{graphicx}
\usepackage{bm}
\usepackage{multirow}
\newcommand{\bs}[1]{\boldsymbol{#1}}
\usepackage{algorithm}
\usepackage{algpseudocode}
\usepackage{tikz}

\usepackage{amsmath}
\usepackage{makecell}

\usepackage{enumitem}

\RequirePackage{xspace}
\makeatletter
\DeclareRobustCommand\onedot{\futurelet\@let@token\@onedot}
\def\@onedot{\ifx\@let@token.\else.\null\fi\xspace}
\def\eg{\emph{e.g}\onedot}
\def\ie{\emph{i.e}\onedot} 
\usepackage{needspace}
\usepackage[misc]{ifsym}

\newcommand\blfootnote[1]{%
  \begingroup
  \renewcommand\thefootnote{}\footnote{#1}%
  \addtocounter{footnote}{-1}%
  \endgroup
}

\usepackage[accsupp]{axessibility}  

\usepackage{hyperref}

\usepackage{orcidlink}

\begin{document}

\title{Distill on a Diet: Efficient Knowledge Distillation via Learnable Data Pruning} 

\titlerunning{Distill on a Diet}

\author{
Yifan Wu\inst{1,2}\textsuperscript{*} \and
Yiqi Wang\inst{4}\textsuperscript{*} \and
Xichen Ye\inst{3}\textsuperscript{*} \and
Wenjing Yan\inst{2} \and
Xiaoqiang Li\inst{4}\textsuperscript{(\Letter)} \and
Cheng Jin\inst{3} \and
Xiangyu Yue\inst{2} \and
Weizhong Zhang\inst{1}\textsuperscript{(\Letter)}
}

\authorrunning{Y.~Wu et al.}

\institute{
School of Data Science, Fudan University \and
Faculty of Engineering, The Chinese University of Hong Kong \and
College of Computer Science and Artificial Intelligence, Fudan University \and
School of Computer Engineering and
Science, Shanghai University\\
\email{victorwu@link.cuhk.edu.hk}\blfootnote{$*$ Equal contribution. \Letter~Corresponding authors.}
\email{\{wangyq2004,xqli\}@shu.edu.cn}
\email{xcye25@m.fudan.edu.cn}
\email{wenjingyan@cuhk.edu.hk}
\email{xyyue@ie.cuhk.edu.hk}
\email{\{jc,weizhongzhang\}@fudan.edu.cn}
\url{https://github.com/yifanwu-victor/Distill-on-a-Diet}\\
}

\maketitle

\begin{abstract}
Knowledge Distillation (KD) is widely used to obtain compact models for efficient inference in resource-constrained environments.
Yet the computational overhead of the distillation process itself is often overlooked, raising the question of whether a better student model can be obtained with less data and less compute via data pruning.
However, existing data pruning methods are not designed for KD: some introduce substantial overhead (\eg, obtaining training dynamics through retraining), while others rely on heuristic selection rules that fail to capture what KD actually requires, often resulting in suboptimal subsets.
To address these issues, we propose \textbf{IF-Beta}, an efficient data pruning framework that combines 
influence function and a learnable sampling policy.
Empirically, we first demonstrate that influence functions can serve as an effective and efficient estimator of sample impact in KD settings, where only a pretrained teacher is available.
Building on this, our sampling policy is specifically parameterized by a Beta distribution, whose highly flexible two-parameter family allows the policy to adapt to diverse pruning regimes rather than being tied to fixed heuristic forms.
Next, we formulate KD pruning as optimizing this policy through a bilevel objective, where the inner loop operates in the teacher’s feature space with a KD-aligned objective, enabling fast proxy training, while the outer loop updates the policy parameters to maximize the distillation performance.
This design ensures IF-Beta is both computationally efficient and inherently aligned with the goals of KD.
Extensive experiments on CIFAR-10/100 and ImageNet show that IF-Beta consistently outperforms other baselines across a wide range of pruning ratios.
Remarkably, IF-Beta does enable students trained on less data and less compute to surpass the performance of students distilled on the full dataset.
\end{abstract}

\section{Introduction}
\label{sec:intro}

Over the past decade, deep neural networks (DNNs)~\cite{DBLP:conf/nips/DCNN,DBLP:conf/cvpr/Resnet,DBLP:conf/iclr/ViT} have achieved remarkable success across a wide range of domains, largely driven by the increasing scale of models and datasets. 
Yet such large-scale models inevitably incur substantial computation cost at inference, which severely limits their applicability in real-world scenarios. 
To address this issue, \cite{DBLP:journals/corr15/KnowledgeDistilling} introduced Knowledge Distillation (KD), enabling compact students to inherit the knowledge of cumbersome teachers and thereby achieve efficient inference with competitive performance~\cite{DBLP:conf/emnlp/TinyBert,gou2021kdsurvey,zhao2022DKD}.
Despite its inference benefits, KD does not alleviate the training burden: students are usually distilled on the entire dataset under the guidance of a large teacher, resulting in significant training cost~\cite{iclr25medium}. 


Several prior methods have been explored to accelerate KD training~\cite{DBLP:conf/eccv/LiangLBZTLF22, DBLP:conf/iccv/YangZLZY023, DBLP:conf/cvpr/BeyerZRMA022, DBLP:conf/eccv/ShenX22}, yet their speedups typically come at the cost of degrading student performance~\cite{iclr25medium}.
Alternatively, an emerging line of research~\cite{baruch2025KDinPruning, iclr25medium} points to another promising direction: reducing the number of samples that participate in distillation, \ie, data pruning~\cite{phillips2017coresets, 2025coresetSurvey}, which aims to retain only a subset of training samples while maintaining competitive performance of the models.
Baruch et al.~\cite{baruch2025KDinPruning} was the first to systematically re-evaluate classical pruning methods~\cite{el2n, toneva2018forgetting, pleiss2020aum} in KD, demonstrating that data pruning can indeed be effective for KD.
More recently, Chen et al.~\cite{iclr25medium} introduced a pruning method specifically tailored for KD, which directly employs the teacher’s cross-entropy (CE) loss as a difficulty metric and selects samples whose losses fall within a ``medium'' range window.

Despite these advances, data pruning in the context of KD remains hindered by two fundamental, yet orthogonal, obstacles:
(i) \textbf{Efficiency-Effectiveness Trade-off in Score Estimation}.
High-fidelity trajectory-based metrics (\eg, EL2N~\cite{el2n}, Forgetting~\cite{toneva2018forgetting}, and AUM~\cite{pleiss2020aum}) are computationally prohibitive for KD, as they depend on training-time statistics that are generally not provided with pretrained teachers.
Consequently, obtaining these scores necessitates expensive retraining, which largely offsets the efficiency gains intended by pruning.
In contrast, the CE loss adopted by Chen et al.~\cite{iclr25medium} does greatly reduce time cost by computing only the per-sample loss.
However, it is an unreliable indicator of sample difficulty, as over-parameterized deep networks are capable of overfitting even hard or mislabeled samples~\cite{DBLP:conf/icml/ArpitJBKBKMFCBL17}, resulting in small yet weakly discriminative losses across samples of widely varying difficulty (more details see \cref{fig: if_sc}).
(ii) \textbf{Rigid and Suboptimal Sampling}.
Even with reliable scores, conventional sampling strategies rely heavily on manual heuristics, such as top-$k$~\cite{baruch2025KDinPruning}, stratified sampling with a fixed cut-off ratio~\cite{DBLP:conf/iclr/ccs}, or window-based selection over the score distribution~\cite{bws, iclr25medium}.
These hand-crafted designs share a common limitation: they impose a fixed selection region that cannot adapt to the varying data distributions and teacher-student dynamics across different KD settings, inevitably leading to limited performance.

To overcome the aforementioned limitations, we propose an optimized KD pruning framework, \textbf{Influence Function based Beta Policy (IF-Beta)}, which unifies IF scoring with a learnable probabilistic pruning policy.
In particular, we first revisit the influence function (IF)~\cite{koh2017understanding}, which is a post-hoc approach to quantify the impact of each training data.
Thanks to the recent advances~\cite{ye2025robust}, we show that IF serves as an efficient and effective score estimator for KD-oriented data pruning, replacing prior score-based approaches.
Building upon this IF-based scoring, we further introduce a learnable probabilistic pruning policy based on a parameterized Beta distribution, called Beta Policy, which provides a shape-adaptive mechanism for selecting data.
Formally, we pose the optimization of our Beta Policy as the following bilevel optimization problem:
\begin{gather}
    \min_{\phi} \Phi(\phi) = \mathbb{E}_{\bm m \sim\pi_{\phi}^r} \mathcal{\widehat{L}} (\theta_S^*(\bm m)), \\
    s.t. \;
    \theta^*_S(\bm m) = \underset{\theta_S}{\arg \min}\, \mathcal{L} (\theta_S; \theta_T,  \bm m),
\end{gather}
where $\phi$ parameterizes the Beta Policy $\pi_\phi^r$ over the scoring space.
The inner-level objective trains the student $\theta_S$ on the subset of samples specified by the mask $\bm m \sim \pi_\phi^r$, using the teacher $\theta_T$ for guidance through the KD loss $\mathcal{L}(\theta_S; \theta_T, \bm m)$.
The outer-level objective then updates $\phi$, evaluated by the validation objective $\widehat{\mathcal{L}}(\theta_S^*)$, to improve the expected distillation performance.
To further realize this efficiently, we simulate the student side of KD using a lightweight linear classifier trained on the frozen teacher features.  
This proxy-student formulation faithfully captures the knowledge transfer dynamics of KD while keeping the overall computational cost negligible.
Overall, we summarize our contributions as follows:
\begin{itemize}
    \item We investigate the influence function as an efficient and effective estimator for KD pruning, and empirically show that IF can substitute both difficulty-based and uncertainty-based scores utilized in conventional data pruning without requiring retraining.
    \item We introduce a learnable data pruning policy based on a parameterized Beta distribution, which can be optimized efficiently via a bilevel formulation in the teacher’s feature space.
    \item Extensive experiments demonstrate that IF-Beta achieves consistently strong performance across KD settings, and show that IF-Beta can surpass full-data distillation using less data and less compute.
\end{itemize}

\section{Preliminaries}
\label{sec:pre}

We first revisit the formulation of knowledge distillation (KD) and influence functions (IF), which together establish the basis in our framework.
A more detailed discussion, along with other related works (including knowledge distillation, data pruning and influence function), is provided in the Appendix~\ref{app: related works}.

\vspace{0.05in}
\noindent\textbf{Knowledge Distillation.}
The goal of KD is to transfer the knowledge of a large teacher network into a compact student model, enabling efficient inference while maintaining competitive performance~\cite{DBLP:journals/corr15/KnowledgeDistilling}.
In this work, we typically focus on the standard KD framework for image classification tasks, which simply leverages the teacher's predictive distribution to guide the optimization of the student. 

Let the training dataset be $D_{\text{tr}} = \{ z_n = (x_n, y_n) \}_{n=1}^N$. 
Given a teacher network $f_T$ and a student network $f_S$, the student is trained on $D_{\text{tr}}$ by minimizing a weighted combination of the standard cross-entropy loss and a distillation term:
\begin{equation}
\ell_\text{KD} = (1-\alpha) \, \ell_{\text{CE}} + \alpha \, \ell_{\text{KL}},
\label{eq:kd_total}
\end{equation}
where $\alpha \in [0,1]$ and the distillation term is the Kullback--Leibler (KL) divergence between the teacher and student predictive distributions:
\(\ell_{\text{KL}} = \text{KL}(f_T(x)\,\|\,f_S(x)).\)

\vspace{0.05in}
\noindent\textbf{Data Pruning.}
Data pruning aims to select a subset of training samples that preserves the performance of the full dataset while reducing computational cost.
Most existing pruning methods~\cite{baruch2025KDinPruning, iclr25medium, DBLP:conf/iclr/ccs, bws} follow a two-stage pipeline: (i) estimating a score for each sample, and (ii) selecting a subset based on these scores under a given pruning ratio.

For score estimation, prior work typically measures either sample difficulty (\eg, EL2N~\cite{el2n}, Forgetting~\cite{toneva2018forgetting}, AUM~\cite{pleiss2020aum}) or predictive uncertainty (\eg, Dyn-Unc~\cite{DynamicUncertainty}, DUAL~\cite{DUAL}), which are derived from training dynamics (\ie, per-epoch predictions) and therefore require full or partial model retraining.
In KD, however, only a pretrained teacher is available, without its training trajectory.
This motivates the need for effective and efficient post-hoc score estimators that can be computed directly from a pretrained model.

Given the estimated scores, existing methods typically apply heuristic selection rules, such as top-$k$~\cite{baruch2025KDinPruning}, stratified sampling with manually designed cut-off ratio~\cite{DBLP:conf/iclr/ccs}, or sliding-window mechanisms~\cite{bws, iclr25medium}.
Such heuristic strategies may limit subset optimality under different pruning regimes, motivating more principled and adaptive selection mechanisms. 

\vspace{0.05in}
\noindent\textbf{Influence Functions.}
To address the limitation of trajectory-dependent scoring, we revisit influence functions as a potential post-hoc estimator, which aims to understand the influence of individual training points on the model’s predictions, \ie, how the model’s output would change if a particular training point were removed~\cite{hampel1974influence,cook1980characterizations}.

Consider a model, whose parameters $\theta^\star$ are obtained via empirical risk minimization (ERM) over $D_{\text{tr}}$:
\begin{equation}
\theta^\star := \arg\min_\theta \frac{1}{N} \sum_{n=1}^N \ell(z_n, \theta),
\end{equation}
where $\ell$ denotes a per-sample loss function, \eg, the cross-entropy loss for classification tasks.

Next consider a validation set $ D_{\text{val}} = \{ z_m = (x_m, y_m) \}_{m=1}^M $.
Following Koh and Liang~\cite{koh2017understanding}, the influence of a training point $z_{\text{tr}} \in D_{\text{tr}}$ on a specific validation example $z_{\text{val}} \in D_{\text{val}}$ can be expressed as:
\begin{equation}
\mathcal{I}(z_\text{tr}; z_\text{val}) := g_{z_\text{val}}^\top H_{\theta^\star}^{-1} g_{z_\text{tr}},
\end{equation}
where $g_z = \nabla_{\theta^\star} \ell(z, \theta^\star)$ denotes the gradient of the loss 
w.r.t. $\theta^\star$ for $z$, 
and $H_{\text{tr}} = \frac{1}{N} \sum_{n=1}^N \nabla^2_{\theta^\star} \ell(z_n, \theta^\star)$ 
is the Hessian matrix computed over $D_{\text{tr}}$.
Therefore, the influence function of $z_\text{tr}$ on $D_\text{val}$ can be extended as follows:
\begin{equation}
    \label{eq: standard influence function}
    \mathcal{I}(z_\text{tr}; D_\text{val})
    := \frac{1}{M} \sum_{m=1}^M g_{z_m}^\top H_\text{tr}^{-1} g_{z_\text{tr}},
\end{equation}
where $z_m \in D_\text{val}$.

This allows us to quantify the contribution of each training point.
Despite its appealing theoretical foundation, traditional IF estimators were long considered impractical due to their high computational cost~\cite{iclr24datainf} and unreliable estimation quality~\cite{iclr21IFisFragile,nips22IFQA}.

\section{Method}
\label{sec:method}

In this section, we first introduce an efficient post-hoc score estimator based on influence functions (\Cref{sec: IF-FVM}). Building upon the estimated scores, we then propose an adaptive sampling strategy via a Beta Policy (\Cref{sec: Beta Policy}).
The search for the optimal policy is formulated as a bilevel optimization problem (\Cref{sec: bilevel optimization}), for which we further present an efficient solution to this optimization in the feature space under the KD setting (\Cref{sec: efficient bilevel}).

\subsection{IF Serves as an Efficient Score Estimator}
\label{sec: IF-FVM}

To address the limitation of trajectory-dependent scoring in KD, we revisit influence functions as a post-hoc estimator.
Although classical influence estimation suffers from instability and high computational cost in deep neural networks, recent advances~\cite{ye2025robust} show that optimizing towards a flat validation minimum (FVM) significantly stabilizes influence estimation (also observed in \cref{fig: if_sc}).

Following this insight, we fine-tune the pretrained teacher model $\theta_T$ on the validation set $\mathcal{D}_\text{val}$ using Sharpness-Aware Minimization (SAM)~\cite{iclr21sam}.
Typically, we solve:
\begin{gather}
    \tilde{\theta}_T
    := \arg \min_{\theta_T} \hat{R}^\gamma_\text{val} (\theta_T),
    \label{eq: SAM} \\
    \text{where}\;
    \hat{R}^\gamma_\text{val} (\theta_T)
    := \max_{\|\Delta\| \le \gamma} 
    \hat{R}_{\text{val}}(\theta_T + \Delta).
\end{gather}
Here, $\gamma$ is a fixed constant and $\hat{R}_\text{val}$ denotes the empirical validation risk.
This fine-tuning step is highly efficient because it operates directly on the given teacher model, introducing only modest computational overhead.
The SAM-style objective improves the stability of curvature estimation by enforcing local flatness of the loss landscape, thereby enabling IF to achieve more reliable influence estimation than previous methods.

Let $\tilde{g}_{z} = \nabla_{\tilde{\theta}_T} \ell(z, \tilde{\theta}_T)$ denote the gradient of the loss w.r.t $\tilde{\theta}_T$ for the sample $z$, and let $\tilde{H}_{\text{val}} = \frac{1}{M} \sum_{m=1}^M \nabla^2_{\tilde{\theta}_T} \ell (z_m, \tilde{\theta}_T)$ represent the Hessian matrix computed over the validation set $D_\text{val}$.
Under the FVM condition, the influence of a training sample $z_\text{tr}$ on the validation set $D_\text{val}$ can then be derived as:
\begin{equation}
    \mathcal{I}(z_\text{tr}, D_\text{val}) := \tilde{g}_{z_\text{tr}}^\top \tilde{H}^{-1}_\text{val} \tilde{g}_{z_\text{tr}},
    \label{eq: IF-FVM}
\end{equation}
which we adopt as the score estimator for sample selection.
A detailed derivation is provided in Appendix~\ref{app:derivation}.

In practice, directly computing Eq.~\ref{eq: IF-FVM} can be inefficient, as it requires estimating the inverse Hessian matrix $\tilde{H}^{-1}_\text{val}$.
To address this issue,
we approximate the inverse Hessian using the inverse of a diagonal Fisher Information matrix, $\mathrm{diag}\left( \frac{1}{|D_\text{tr}|} \sum_{z_\text{tr} \in D_\text{tr}}\tilde{g}_{z_\text{tr}} \tilde{g}_{z_\text{tr}}^\top \right)$,
which yields a highly efficient approximation and makes the overall score computation practical at scale.


\begin{wrapfigure}{r}{0.55\linewidth}
  \vspace{-0.3in}
  \centering
    \raisebox{0mm}{\includegraphics[width=0.057\linewidth]{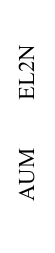}}%
    \subcaptionbox{CIFAR-10\label{fig:cifar10_if}}{\includegraphics[width=0.4\linewidth]{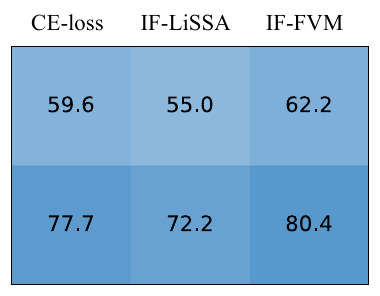}}%
    \hspace{1pt}
    \subcaptionbox{CIFAR-100\label{fig:cifar100_if}}{\includegraphics[width=0.4\linewidth]{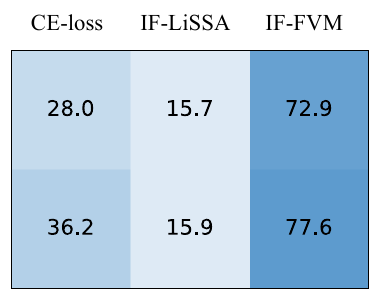}}%
    \hspace{0pt}%
    \raisebox{0.4mm}{\includegraphics[width=0.055\linewidth]{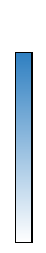}}%
    \vspace{-0.05in}
  \caption{
    Spearman rank correlation (\%) between post-hoc score estimators and trajectory-based difficulty metrics.
    Higher values indicate stronger alignment in sample ranking.
  }\label{fig: if_sc}
  \vspace{-0.25in}
\end{wrapfigure}
\cref{fig: if_sc} validates the reliability of our proposed IF-FVM score estimator.
Specifically, we benchmark various post-hoc score estimators by computing the Spearman rank correlation between their scores and two established trajectory-based difficulty metrics: EL2N~\cite{el2n} and AUM~\cite{pleiss2020aum}.
These metrics are widely regarded as faithful proxies for sample hardness.
Intuitively, a higher correlation coefficient indicates that the estimator more accurately captures the intrinsic difficulty of samples by effectively aligning with the learning dynamics observed during training.
As shown, IF-FVM consistently maintains a strong correlation across datasets, outperforming both the CE-loss baseline~\cite{iclr25medium} and IF-LiSSA~\cite{jmlr17lissa}.
Our results suggest that IF-FVM serves as a computationally efficient yet robust alternative to metrics that otherwise require expensive trajectory tracking.

\subsection{Adaptive Sampling via Beta Policy}
\label{sec: Beta Policy}




  


Having established that IF under FVM serves as an effective score estimator, we next pose a natural question:
\textit{Can existing sampling strategies for sample selection be directly applied to KD scenarios without compromising performance or efficiency?}
While applicable in principle, their effectiveness in KD remains suboptimal.
Specifically, we identify critical limitations in representative methods, CCS~\cite{DBLP:conf/iclr/ccs} and BWS~\cite{bws}, within KD contexts:
CCS exhibits a threshold shift, where the optimal cutoff ratio varies significantly between KD and non-KD settings, rendering pre-defined hyperparameters non-transferable (\cref{fig:ccs_wo_kd,fig:ccs_w_kd}).
Meanwhile, BWS is hindered by a suboptimal fixed-window design; as our replacement experiments (\cref{fig:window_is_not_good}) reveal, introducing randomness into the selection window actually improves performance.
These findings (further detailed in Appendix~\ref{app:analysis_sampling}) underscore the need for an adaptive, learnable sampling strategy that overcomes the rigidity of current heuristic-based selection.

\begin{figure}[t]
\centering
\vspace{-0.1in}

\begin{subfigure}[t]{0.27\linewidth}
  \centering
  \includegraphics[height=2.6cm]{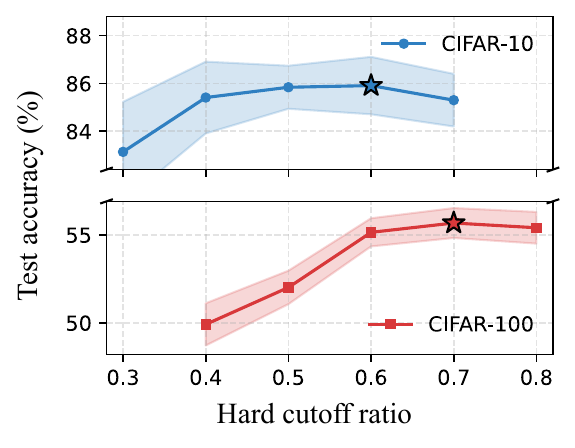}
  \caption{CCS w/o KD.}
  \label{fig:ccs_wo_kd}
\end{subfigure}
\hfill
\begin{subfigure}[t]{0.27\linewidth}
  \centering
  \includegraphics[height=2.6cm]{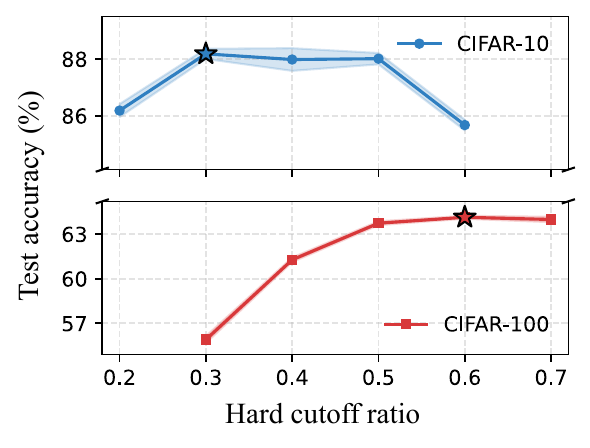}
  \caption{CCS w/ KD.}
  \label{fig:ccs_w_kd}
\end{subfigure}
\hfill
\begin{subfigure}[t]{0.42\linewidth}
  \centering
  \includegraphics[height=2.6cm]{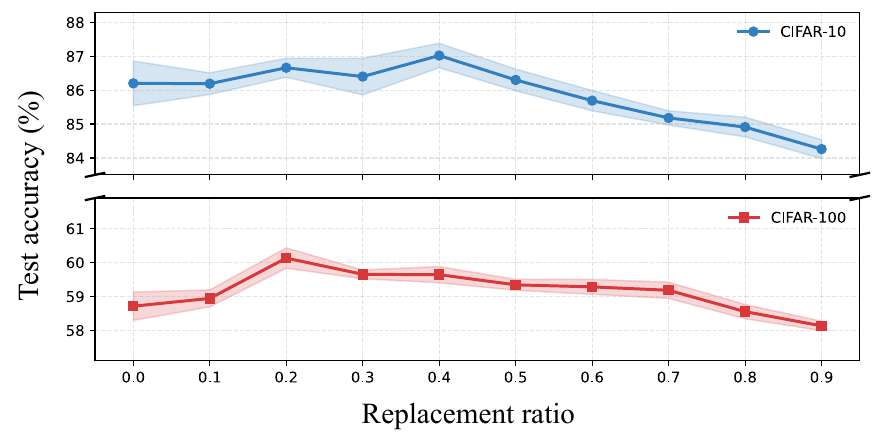}
  \caption{Best-window w/ replacement.}
  \label{fig:window_is_not_good}
\end{subfigure}

\vspace{-0.05in}
  \caption{
    Limitations of heuristic sampling methods on ResNet-18 with a pruning ratio 90\%.
    (a,b) CCS with different hard cutoff ratios when w/o KD and W/ KD.
    (c) BWS with different replacement ratios with outside samples.
  }
\label{fig:ccs_bws_limitaion}
\vspace{-0.15in}
\end{figure}

\begin{wrapfigure}{r}{0.5\linewidth}
    \vspace{-0in}
    \centering
    \includegraphics[width=\linewidth]{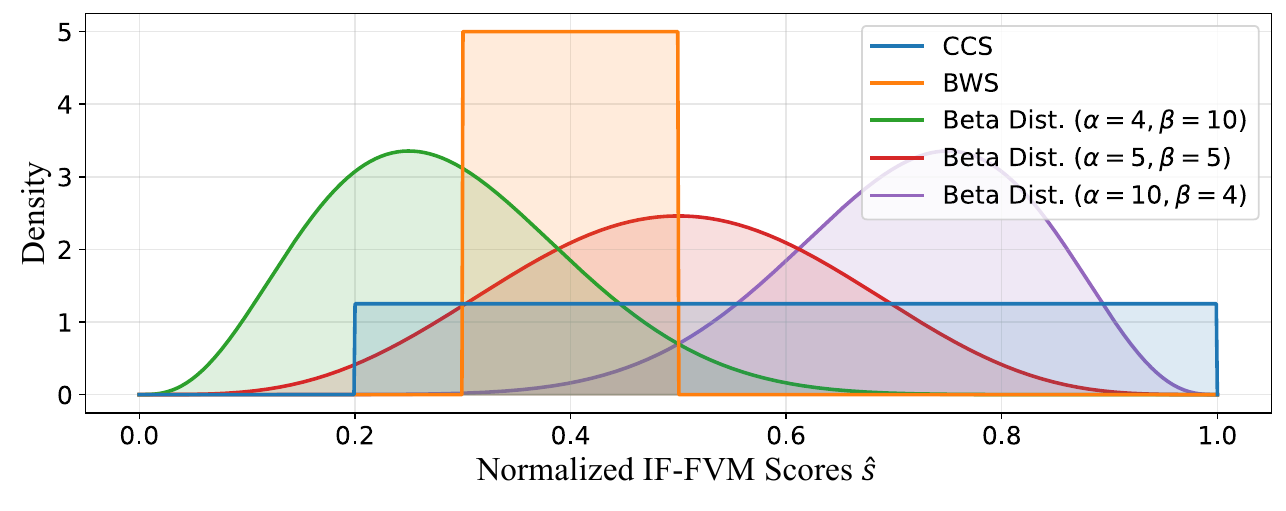}
    \vspace{-0.3in}
    \caption{
        Comparison of our Beta Policy with CCS and BWS.
        Our Beta Policy provides a more flexible and adaptive sampling distribution.
    }
    \label{fig: beta_dist}
    \vspace{-0.25in}
\end{wrapfigure}
To address the aforementioned limitations, we propose a Beta-based sampling policy.
Given a precomputed difficulty score $s_i$ for each sample $z_i$, we apply rank-to-percentile normalization to map the scores onto the interval $[0,1]$, obtaining normalized values $\hat{s}_i \in [0, 1]$, where smaller values correspond to more difficult samples.
Accordingly, we parameterize the selection probability of each sample using the Beta distribution:
\begin{equation}
    p_i(\phi) = \frac{1}{Z} \frac{1}{B(\alpha(\phi), \beta(\phi))}
    \,\hat{s}_i^{\alpha(\phi)-1}(1-\hat{s}_i)^{\beta(\phi)-1},
    \label{eq:beta_pdf}
\end{equation}
where $B(\cdot, \cdot)$ is the Beta function, and $Z$ is a normalization constant ensuring $\sum_i p_i=1$.
Here, $\phi$ represents the learnable parameters that jointly control $\alpha$ and $\beta$.
As illustrated in \cref{fig: beta_dist}, our approach provides a continuous and flexible density over the sample space.
This contrasts with the rigid, uniform sampling over predefined support sets used in CCS and BWS, allowing for a more adaptive selection strategy.

Based on the parameterized sampling probability $p_i(\phi)$ defined in Eq.~\ref{eq:beta_pdf}, we introduce the \textbf{Beta Policy} $\pi_{\phi}^r$, a Beta distribution-based categorical distribution over the training samples conditioned on $\phi$ and $r$.
This policy enables sampling a binary mask vector $\bm m \in \{0, 1\}^{N}$ with a fixed pruning ratio $r \in [0, 1]$ according to
\begin{equation}
    \bm m \sim \pi_\phi^r := p(\bm m \mid \phi, r),
\end{equation}
where
\begin{equation}
    p(\bm m \mid \phi, r) = \frac{1}{Z_{K_r}} \prod_{i=1}^{N} p_i(\phi)^{m_i} 1_{\|\bm m\|_1 = K_r}(\bm m),
\end{equation}
$1_{{\|\bm m\|_1 = K_r}}(\cdot)$ denotes the indicator function and $Z_{K_r}$ is the normalization constant ensuring a valid categorical distribution over all masks satisfying $\|\bm m\|_0 = K_r = (1-r)N$.
The sampled mask vector $\bm m$ then defines a pruned subset of the training data, $\mathcal{D}_\text{sub} = \{z_i \mid m_i = 1, z_i \in \mathcal{D}_\text{tr}\}$.

\subsection{Bilevel Optimization for the Beta Policy}
\label{sec: bilevel optimization}

Following Zhou et al.~\cite{icml22BilevelCoreset}, we formulate the search for the optimal policy $\pi_{\phi}^r$ as a bilevel optimization problem:
\begin{gather}
    \min_{\phi} \Phi(\phi) = \mathbb{E}_{\bm m \sim\pi_{\phi}^r} \mathcal{\widehat{L}} (\theta_S^*(\bm m)), \label{obj: outer} \\
    s.t. \;
    \theta^*_S(\bm m) = \underset{\theta_S}{\arg \min}\, \mathcal{L} (\theta_S; \theta_T,  \bm m). \label{obj: inner}
\end{gather}
Here, $\theta_S$ denotes the model parameter of the student model.
The inner objective evaluates the training risk of $\theta$ on the sampled subset $\mathcal{D}_\text{sub}$ using the Knowledge Distillation loss $\ell_\text{KD}$:
\begin{equation}
    \mathcal{L}(\theta_S; \theta_T, \bs{m}) = \frac{1}{K} \sum_{z_i \in D_{\text{sub}}} \ell_\text{KD}(z_i, \theta_S, \theta_T),
    \label{eq: inner obj}
\end{equation}
where $\theta_T$ represents the parameters of the teacher model.
The outer objective measures the validation risk of the optimized model $\theta_S^*(\bm m)$ on the validation set $\mathcal{D}_\text{val}$ using Cross Entropy loss $\ell_\text{CE}$:
\begin{equation}
    \widehat{\mathcal{L}}(\theta_S^{*}(\bs{m})) = \frac{1}{M} \sum_{z_i \in D_{\text{val}}} \ell_\text{CE} \big(z_i, \theta_S^*(\bs{m})\big).
    \label{obj: outer obj}
\end{equation}
Intuitively, this bilevel formulation seeks the optimal policy such that the sampled subset induces a model that generalizes best on the validation set.

However, directly solving this bilevel optimization yields a gradient estimate of the form
\(
\nabla_{\phi} \Phi(\phi) \approx 
\nabla_{\phi}\theta_S^*(\bs{m}) 
\nabla_{\theta_S}\widehat{\mathcal{L}}(\theta_S^*(\bs{m}))  
\),
which requires implicit differentiation of the inner optimum, \ie, $\nabla_{\phi}\theta_S^*(\bs{m})$, and is thus computationally prohibitive for deep neural networks.
By exploiting the probabilistic formulation of our Beta Policy $\pi_{\phi}^r$, we circumvent this issue by adopting a policy gradient estimator, thereby eliminating the need for costly implicit differentiation.

Particularly, we establish the following result:
\begin{theorem}[Unbiased Policy-Gradient Estimator]
    \label{thm:unbiased-pg}
    Let $\pi_\phi^r$ be a differentiable sampling policy over mask vectors $\bm m \in \{0, 1\}^N$.
    Define the policy-gradient estimator
    \begin{equation}
    \hat{g}(\bm m) = \widehat{\mathcal{L}} \left( \theta_S^{*}(\bs{m}) \right) \nabla_{\phi} \ln p(\bm m | \phi, r).
    \label{eq: policy gradient}
    \end{equation}
    Then $\hat{g}(\bm m)$ is an unbiased estimator of the true gradient $\nabla_\phi \Phi(\phi)$:
    \begin{equation}
        \mathbb{E}_{\bm m \sim \pi_\phi^r} \left[ \hat{g}(\bm m) \right]
        = \nabla_\phi \Phi(\phi).
    \end{equation}
\end{theorem}
The proof is provided in Appendix~\ref{app:ori_PGE}.
\cref{thm:unbiased-pg} ensures that $\hat{g}(\bm m)$ provides an unbiased stochastic gradient for optimizing $\phi$.
Consequently, given the inner-loop optimum $\theta_S^{*}(\bs{m})$, the policy parameter $\phi$ can then be updated via stochastic gradient descent:
\begin{align}
    \phi \leftarrow  \phi - \eta \, \widehat{\mathcal{L}}\left(\theta_S^{*}(\bs{m})\right) \, \nabla_{\phi}\ln  p(\bs{m} | \phi, r), \label{eq: SGD}
\end{align}
where $\eta$ denotes the learning rate.

\subsection{Efficient Bilevel Optimization in Feature Space}
\label{sec: efficient bilevel}

While the outer-loop objective in Eq.~\ref{obj: outer} can be efficiently optimized via the policy gradient estimator (Eq.~\ref{eq: SGD}), the inner-loop optimization in Eq.~\ref{obj: inner} remains computationally expensive, as it requires retraining the student network to convergence for every policy update.

To address this issue, we leverage the structure of the KD setting.
Remarkably, instead of training the student model’s full parameter set $\theta$ from scratch, we utilize a surrogate student
by attaching a trainable linear classification head $f_L$ to the fixed pretrained teacher model with parameters $\theta_T$.
Let $f_T(\cdot)$ denote the fixed feature extractor parameterized by $\theta_T$.
Each training sample $x_i$ is thus represented by its frozen feature $\mathbf{f}_i = f_T(x_i) \in \mathbb{R}^d$.
Accordingly, the student model is formulated as:
\begin{equation}
    f_{\theta_S}(x_i) := f_L(\mathbf{f}_i) = \bs{C}\,\mathbf{f}_i,
\end{equation}
where $\bs{C} \in \mathbb{R}^c \times \mathbb{R}^d$ denotes the parameters of the linear classifier that maps feature vectors to class logits.

Finally, the inner-loop objective in Eq.~\ref{eq: inner obj} can be reformulated as
\begin{equation}
    \bm C^*(\bm m) = \underset{\bm C}{\arg \min} \frac{1}{K} \sum_{z_i \in D_{\text{sub}}} \ell_\text{KD}(z_i, \bm C, \theta_T),
    \label{eq: easy inner obj}
\end{equation}
which can be efficiently optimized since it involves learning only a single linear layer.
Further detailed implementation and the complete pipeline of IF-Beta are provided in Appendix~\ref{app:implementation}.

\vspace{0.05in}
\noindent\textbf{Remark.}
It is worth noting that the bilevel formulation in Eq.~\ref{eq: inner obj} presented in \Cref{sec: bilevel optimization} is general and can be directly applied to both KD scenarios~\cite{DBLP:journals/corr15/KnowledgeDistilling} and standard supervised classification~\cite{cifar,DBLP:conf/cvpr/imagenet}.
In particular, by replacing the inner objective in Eq.~\ref{eq: inner obj} with a standard CE loss, our approach naturally transfers to a general data pruning framework.
To demonstrate this versatility, we further conduct standard data pruning experiments following prior works~\cite{DBLP:conf/iclr/ccs,bws,DUAL}, with the results reported in \Cref{sec: standard_data_pruning}.

\section{Experiments}
\label{sec:exp}

In this section, we present extensive experiments to evaluate the effectiveness and generality of IF-Beta.
We first compare existing data pruning methods for distillation across three representative datasets (CIFAR-10/100~\cite{cifar} and ImageNet~\cite{DBLP:conf/cvpr/imagenet}) under various pruning ratios, and with different teacher architectures (\eg, ResNet~\cite{DBLP:conf/cvpr/Resnet}, WideResNet~\cite{zagoruyko2016wideresnet}, ViT~\cite{DBLP:conf/iclr/ViT}).

\vspace{0.05in}
\noindent\textbf{Baselines.}
The baselines considered in this work can be grouped into two categories:
(1) methods that require partial or full retraining to obtain training dynamics, including EL2N~\cite{el2n}, GraNd~\cite{el2n}, Forgetting~\cite{toneva2018forgetting}, and DUAL~\cite{DUAL}; and
(2) methods that do not require retraining, including Random~\cite{baruch2025KDinPruning} and Medium-Difficulty~\cite{iclr25medium}.
In addition, to fairly compare with recent score-based sampling strategies (CCS~\cite{DBLP:conf/iclr/ccs} and BWS~\cite{bws}), we implement them with our introduced IF-based scores, resulting in IF-CCS and IF-BWS, which also do not require retraining.

\vspace{0.05in}
\noindent\textbf{Training Details.}
For distillation, all models are trained with the loss in Eq.~\ref{eq:kd_total}.
Noting that Chen et al.~\cite{iclr25medium} also introduce a new distillation loss, we additionally report results under this loss and denote it as Medium-Difficulty*.
Due to computational constraints, on ImageNet we only compare against EL2N and the six baselines that do not require retraining (including Medium-Difficulty*).
For CIFAR-10/100, all models (across all pruning ratios) are trained with a batch size of 128 for 40K iterations (about 200 epochs over the entire dataset), while for ImageNet all student models are trained with a batch size of 128 for 300K iterations (about 30 epochs over the entire dataset).
All results are averaged over three random seeds (0, 1, 2).
Additional experimental details are provided in Appendix~\ref{app:exp_details}.

\begin{figure*}[tbp]
\centering

\includegraphics[width=\linewidth]{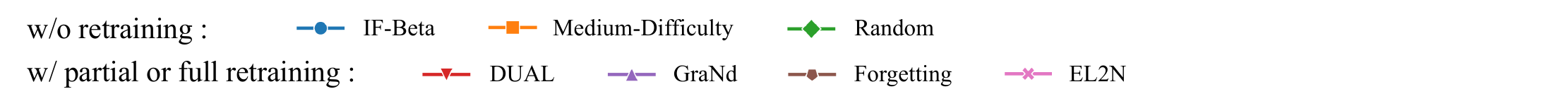}
\\
\vspace{-0.05in}
\raisebox{5pt}{
\begin{minipage}{0.035\linewidth}
  \centering
  \includegraphics[width=\linewidth]{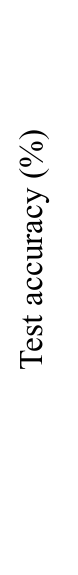}
\end{minipage}
}
\hspace{-0.1in}
\begin{minipage}{0.3\linewidth}
  \centering
  \subcaptionbox{CIFAR-10\label{fig: main_cifar10}}{%
    \includegraphics[width=\linewidth]{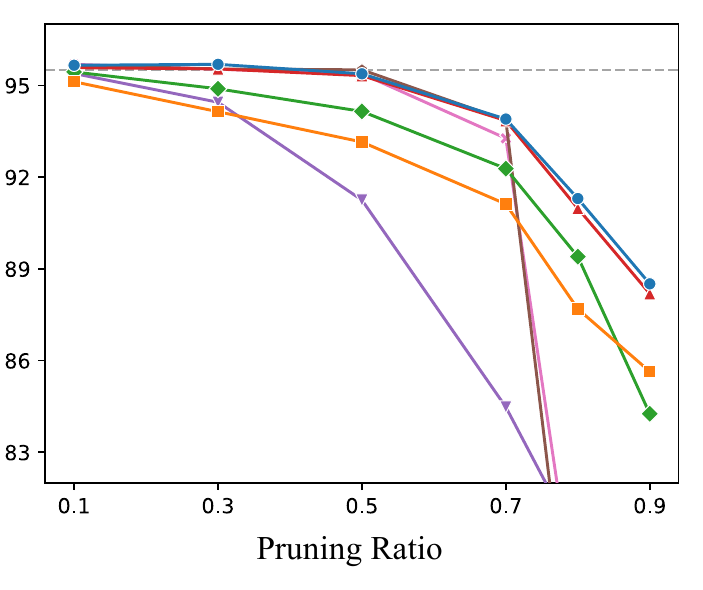}}
\end{minipage}
\hspace{0.02in}
\begin{minipage}{0.3\linewidth}
  \centering
  \subcaptionbox{CIFAR-100\label{fig: main_cifar100}}{%
    \includegraphics[width=\linewidth]{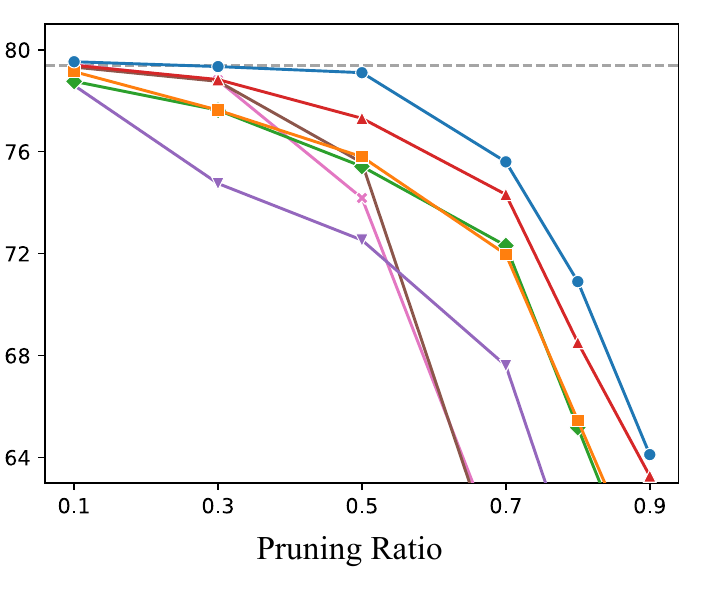}}
\end{minipage}
\hspace{0.02in}
\begin{minipage}{0.3\linewidth}
  \centering
  \subcaptionbox{ImageNet\label{fig: main_imagenet}}{%
    \includegraphics[width=\linewidth]{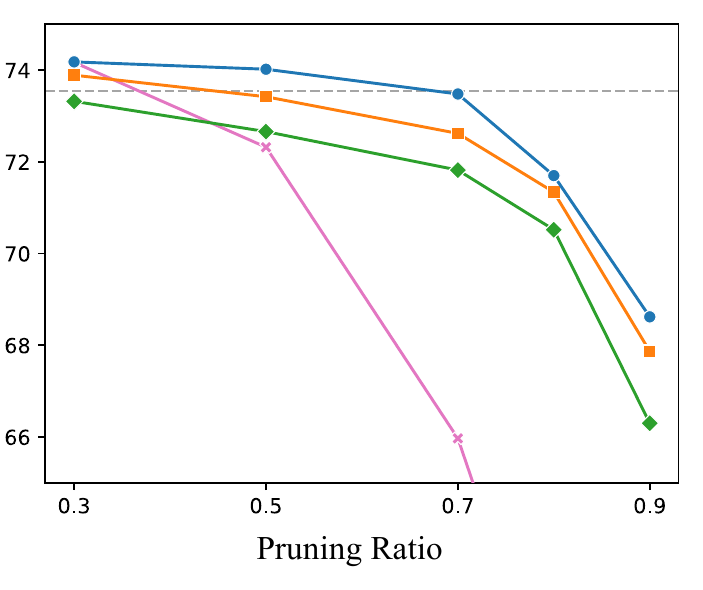}}
\end{minipage}
\vspace{-0.1in}
\caption{Performance comparison between IF-Beta and other baselines on CIFAR-10/100 and ImageNet under the KD setting, where for CIFAR-10/100 both teacher and student are ResNet-18, and for ImageNet both are ResNet-50. The pruning ratio is the fraction of examples removed from the original datasets. The dashed horizontal line denotes the student distilled on the full dataset (without pruning). Detailed numerical results are provided in Appendix~\ref{app:exp}.}
\vspace{-0.1in}
\label{fig: main_kd}
\end{figure*}

\subsection{Main Results: Data Pruning for Knowledge Distillation}
\label{sec: exp_main_kd}

We first evaluate data pruning under the KD setting where the teacher and student 
share the same architecture.
As shown in \cref{fig: main_kd}, IF-Beta consistently outperforms all baselines 
across all datasets and pruning ratios.
Notably, even compared to the recent method DUAL~\cite{DUAL}, which requires partial retraining (60 epochs) to extract training dynamics, IF-Beta consistently achieves superior accuracy.
Compared with recent Medium-Difficulty~\cite{iclr25medium}, which also does not require retraining, IF-Beta even exhibits a clear performance advantage over it.

Remarkably, IF-Beta can surpass full-data distillation using only a subset of the training data: retaining 70\% on CIFAR-10 achieves 95.69\% vs.\ 95.50\%, 90\% on CIFAR-100 reaches 79.53\% vs.\ 79.38\%, and 50\% on ImageNet already yields 74.02\% vs.\ 73.54\%.
This suggests that in KD, not all training samples are equally beneficial---some may actively impede student generalization, and carefully pruning them away yields a cleaner and more informative training signal.

Finally, IF-Beta is further validated under heterogeneous teacher architectures 
spanning ResNet-50/101~\cite{DBLP:conf/cvpr/Resnet}, WideResNet~\cite{zagoruyko2016wideresnet}, 
and ViT-Base~\cite{DBLP:conf/iclr/ViT}, consistently outperforming all 
retraining-free baselines across all settings (see Appendix~\ref{app:exp}).

\subsection{Further Analysis of IF-Beta}

Beyond the main results, we further examine IF-Beta from four complementary perspectives: (i) sample efficiency under constrained training budgets, (ii) end-to-end compute savings relative to full-data training, (iii) advantages over other efficient KD methods, and (iv) generality across diverse distillation objectives.

\begin{table}[tb]
\centering
\caption{Performance comparison between full-data distillation, IF-Beta, and Random
(pruning 10\%, 30\%, 50\%) under limited training budgets on CIFAR-10/100
(10K iterations) and ImageNet (150K iterations).
\textbf{Bold} denotes the best results.}
\label{tab: budget_overview}
\vspace{-0.05in}
\setlength{\tabcolsep}{4pt}
\scriptsize
\begin{tabular}{llccc}
\toprule
Prune & Method & CIFAR-10 & CIFAR-100 & ImageNet \\
\midrule
0\% & Full-data
& 86.62{\tiny $\pm$ 1.67} & 63.54{\tiny $\pm$ 2.21} & 57.16{\tiny $\pm$ 0.14} \\
\midrule
\multirow{2}{*}{10\%}
& IF-Beta
& \textbf{87.81}{\tiny $\pm$ 0.58} & \textbf{66.26}{\tiny $\pm$ 0.91} & -- \\
& Random
& 87.30{\tiny $\pm$ 0.62} & 64.95{\tiny $\pm$ 0.42} & -- \\
\midrule
\multirow{2}{*}{30\%}
& IF-Beta
& \textbf{89.23}{\tiny $\pm$ 1.05} & \textbf{66.53}{\tiny $\pm$ 1.11} & \textbf{61.35}{\tiny $\pm$ 0.09} \\
& Random
& 88.03{\tiny $\pm$ 0.83} & 63.50{\tiny $\pm$ 2.42} & 57.48{\tiny $\pm$ 0.17} \\
\midrule
\multirow{2}{*}{50\%}
& IF-Beta
& \textbf{90.80}{\tiny $\pm$ 1.29} & \textbf{67.31}{\tiny $\pm$ 1.52} & \textbf{60.70}{\tiny $\pm$ 0.12} \\
& Random
& 86.37{\tiny $\pm$ 0.85} & 63.39{\tiny $\pm$ 1.29} & 57.12{\tiny $\pm$ 0.15} \\
\bottomrule
\end{tabular}
\vspace{-0.1in}
\end{table}

\vspace{0.05in}
\noindent\textbf{Sample Efficiency under Constrained Budgets.}
As shown in \cref{tab: budget_overview}, we evaluate IF-Beta under severely 
constrained training budgets (10K iterations for CIFAR-10/100 and 150K for 
ImageNet). IF-Beta consistently outperforms both full-data distillation and 
random pruning across all pruning ratios and datasets---with 50\% data pruned, 
IF-Beta achieves 90.80\% on CIFAR-10, 67.31\% on CIFAR-100, and 60.70\% on 
ImageNet, all exceeding the full-data baselines.
Crucially, the consistent margin over random pruning demonstrates that these 
gains are not solely attributable to a more favorable training regime induced 
by fewer samples,\footnote{With a fixed iteration budget, training on a smaller 
subset effectively increases the number of passes over each sample, which alone 
can improve generalization.} but stem from effective sample selection itself.

\begin{table}[tbp]
\centering
\caption{Compute comparison between IF–Beta and full–data training when KD.
Relative cost denotes the ratio of time between: (i) the total number of iterations IF–Beta actually consumes until surpassing the final full–data students' performance (including the pruning process), and (ii) the standard full–data training schedule.
\textbf{Bold} denotes the best results.
}
\setlength{\tabcolsep}{4pt}
\scriptsize
\vspace{-0.05in}
\begin{tabular}{lccccc}
\toprule
Dataset & Prune & Relative Cost & \textbf{IF-Beta} & \textbf{Full Data} \\
\midrule
CIFAR-10   & 30\% & 0.806$\times$ & \textbf{95.64} & 95.50 \\
CIFAR-100  & 10\% & 0.845$\times$ & \textbf{79.45} & 79.38 \\
ImageNet    & 50\% & 0.915$\times$ & \textbf{73.67} & 73.54 \\
\bottomrule
\end{tabular}
\label{table:less_data_less_time}
\vspace{-0.1in}
\end{table}

\vspace{0.05in}
\noindent\textbf{End-to-End Compute Savings.}
We then examine how much time is actually needed (with less data) to reach the final performance of the full–data student.
As shown in \cref{table:less_data_less_time}, across all three datasets, IF–Beta can surpass the final full–data student using fewer samples and in much lower time cost, even when the cost of IF–Beta pruning process is included.

\begin{table}[t]
\centering
\caption{Test accuracy and training time on ImageNet of prior efficient KD methods and our IF-Beta using 70\% of samples, where the teacher is ResNet-34 and the student is MobileNet. \textbf{Bold} denotes the best results and \underline{underline} denotes the second-best.}
\vspace{-0.05in}
\setlength{\tabcolsep}{4pt}
\scriptsize
\begin{tabular}{c|ccccc}
\toprule
& Baseline~\cite{mobilenet} & Zipf's LS~\cite{DBLP:conf/eccv/LiangLBZTLF22} & USKD~\cite{DBLP:conf/iccv/YangZLZY023} & \makecell{Medium-\\Difficulty*}~\cite{iclr25medium} & IF-Beta \\
\midrule
Acc (\%)
& 69.57 & 69.59 & 70.38 & \underline{70.92} & \textbf{71.20} \\
\midrule
Time (h)
& 39.86 & $>39.86$ & $>39.86$ & \underline{36.98} & \textbf{19.24} \\
\bottomrule
\end{tabular}
\vspace{-0.05in}
\label{table:skd_compare}
\end{table}

\vspace{0.05in}
\noindent\textbf{Comparison with Efficient KD Methods.}
Following the evaluation protocol of Chen et al.~\cite{iclr25medium}, we compare IF-Beta against other efficient KD methods by distilling ResNet-34 into MobileNet on ImageNet using a single NVIDIA A100 GPU.
As shown in \cref{table:skd_compare}, IF-Beta achieves the best accuracy of 71.20\% while requiring only 19.24 hours---less than half the cost of the full-data Baseline (39.86h) and Medium-Difficulty* (36.98h).
The efficiency gains stem from three factors: reduced data volume (70\% subset), no repeated teacher queries per epoch, and a higher-quality subset that preserves distillation effectiveness even under one-shot teacher logits.
Further detailed discussion is provided in Appendix~\ref{app:comparison_other_kd}.


\begin{table}[tb]
\centering
\caption{Performance under different KD losses and pruning ratios on CIFAR-10/100. Med-Dif denotes Medium Difficulty. \textbf{Bold} denotes the best results}
\label{tab:loss_prune_comp}
\vspace{-0.05in}
\resizebox{0.88\columnwidth}{!}{
\setlength{\tabcolsep}{3pt}
\scriptsize
\begin{tabular}{ll|cc|cc|cc}
\toprule
\multirow{2}{*}{Prune} & \multirow{2}{*}{Method}
& \multicolumn{2}{c|}{Relational KD} 
& \multicolumn{2}{c|}{FitNets} 
& \multicolumn{2}{c}{Attention Transfer} \\
& & CIFAR-10 & CIFAR-100 & CIFAR-10 & CIFAR-100 & CIFAR-10 & CIFAR-100 \\
\midrule
0\% & Full-data
& 95.52 & 78.56 & 95.35 & 78.32 & 95.70 & 78.77 \\
\midrule
\multirow{3}{*}{30\%}
& Random   & 95.07 & 76.19 & 94.57 & 75.08 & 94.52 & 75.63 \\
& Med-Dif*  & 94.87 & 77.04 & 93.97 & 75.53 & 94.02 & 75.77 \\
& IF-Beta  & \textbf{95.24} & \textbf{78.59} & \textbf{95.05} & \textbf{77.94} & \textbf{95.34} & \textbf{78.52} \\
\midrule
\multirow{3}{*}{50\%}
& Random   & 94.64 & 74.09 & 93.65 & 71.46 & 93.65 & 71.73 \\
& Med-Dif*  & 94.11 & 74.26 & 92.79 & 71.33 & 92.96 & 72.65 \\
& IF-Beta  & \textbf{95.67} & \textbf{77.82} & \textbf{95.16} & \textbf{77.56} & \textbf{95.04} & \textbf{77.58} \\
\midrule
\multirow{3}{*}{70\%}
& Random   & 93.70 & 69.21 & 91.54 & 64.99 & 92.06 & 66.08\\
& Med-Dif*  & 93.18 & 69.54 & 90.49 & 65.04 & 91.22 & 66.69 \\
& IF-Beta  & \textbf{94.53} & \textbf{73.36} & \textbf{93.49} & \textbf{71.52} & \textbf{93.64} & \textbf{72.26} \\
\bottomrule
\end{tabular}
}
\vspace{-0.1in}
\end{table}

\vspace{0.05in}
\noindent\textbf{Generality Across Distillation Objectives.}
To verify that IF-Beta is not tied to a specific distillation loss, we evaluate it under three representative KD objectives on CIFAR-10/100: Relational KD~\cite{DBLP:conf/cvpr/RKD} (relation-based), FitNets~\cite{DBLP:journals/corr/fitnets} (feature-based), and Attention Transfer~\cite{DBLP:conf/iclr/attentiontransfer} (attention-based).
As shown in \cref{tab:loss_prune_comp}, IF-Beta consistently outperforms both 
Random pruning and Medium-Difficulty across all losses, pruning ratios, and 
datasets, demonstrating its effectiveness and generality.

\vspace{0.05in}
\noindent\textbf{Remark.} 
Taken together, these results suggest a broader perspective on efficient KD: the \textit{quality} of the training data may matter more than the \textit{specifics} of the distillation algorithm. 
Across diverse training budgets, distillation objectives, and teacher-student pairs, IF-Beta consistently benefits from selecting a more informative subset---suggesting that careful data curation is a universally effective and complementary strategy for improving both the efficiency and effectiveness of knowledge distillation.

\begin{table}[t]
\centering
\caption{Performance comparison of different sampling strategies (IF-Beta vs. IF-CCS and IF-BWS) under the KD setting across various pruning ratios on CIFAR-10/100 and ImageNet. \textbf{Bold} denotes the best results and \underline{underline} denotes the second-best.}
\setlength{\tabcolsep}{4pt}
\scriptsize
\vspace{-0.05in}
\begin{tabular}{lccccccc}
\toprule
Method     & 90\% & 80\% & 70\% & 50\% & 30\% & 10\% \\
\midrule
\multicolumn{7}{c}{CIFAR-10} \\
IF-CCS & \underline{88.18} & \underline{90.87} & \underline{93.39} & 93.77 & 94.65 & 95.40 \\
IF-BWS & 86.20 & 86.50 & 89.80 & \underline{94.82} & \underline{95.51} & \underline{95.47} \\
IF-Beta & \textbf{88.51} & \textbf{91.30} & \textbf{93.90} & \textbf{95.38} & \textbf{95.69} & \textbf{95.66} \\
\midrule
\multicolumn{7}{c}{CIFAR-100} \\
IF-CCS & \underline{63.73} & \underline{70.42} & \underline{72.22} & \underline{77.64} & 78.47 & 78.58 \\
IF-BWS & 58.71 & 63.91 & 71.75 & 75.81 & \underline{78.85} & \underline{78.95} \\
IF-Beta & \textbf{64.11} & \textbf{70.90} & \textbf{75.60} & \textbf{79.10} & \textbf{79.34} & \textbf{79.53} \\
\midrule
\multicolumn{7}{c}{ImageNet} \\
IF-CCS & 65.73 & 68.31 & 71.68 & 72.98 & 73.32 & -\\
IF-BWS & \underline{67.90} & \underline{70.56} & \underline{72.86} & \underline{73.63} & \underline{73.66} & -\\
IF-Beta & \textbf{68.62} & \textbf{71.70} & \textbf{73.48} & \textbf{74.02} & \textbf{74.18} & - \\
\bottomrule
\end{tabular}
\vspace{-0.1in}
\label{table:if_bws_ccs_beta}
\end{table}

\subsection{Ablation Study: Advantages of the Learnable Beta Policy}

We also examine the performance of CCS~\cite{DBLP:conf/iclr/ccs} and BWS~\cite{bws} when combined with our influence scores (\ie, IF-CCS and IF-BWS) on CIFAR-10/100 and ImageNet. 
As shown in \cref{table:if_bws_ccs_beta}, IF-Beta consistently outperforms both 
variants across all pruning ratios and datasets, demonstrating that the learnable 
Beta Policy provides clear advantages over fixed-threshold and fixed-window 
selection, consistent with our analysis in \Cref{sec: Beta Policy}.

\subsection{Data pruning for non-KD Scenario}
\label{sec: standard_data_pruning}

\begin{figure*}[tbp]
\centering

\includegraphics[width=0.7\linewidth]{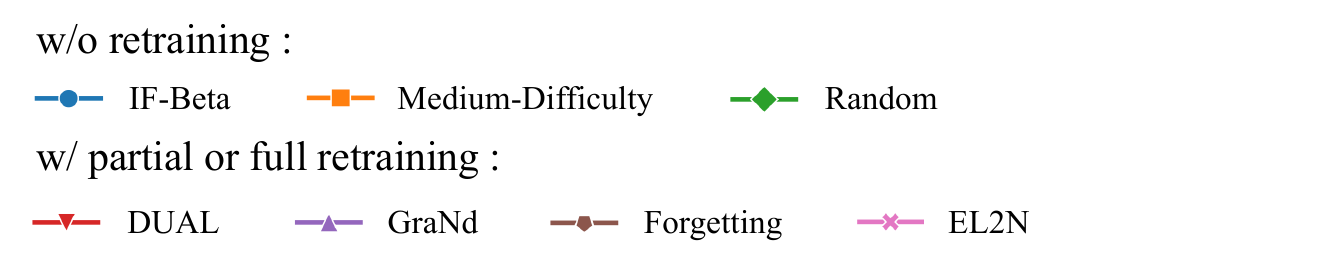}
\\
\vspace{-0.05in}
\raisebox{6pt}{
\begin{minipage}{0.036\linewidth}
  \centering
  \includegraphics[width=\linewidth]{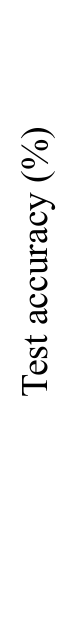}
\end{minipage}
}
\hspace{-0.1in}
\begin{minipage}{0.35\linewidth}
  \centering
  \subcaptionbox{CIFAR-10\label{fig: std_cifar10}}{%
    \includegraphics[width=\linewidth]{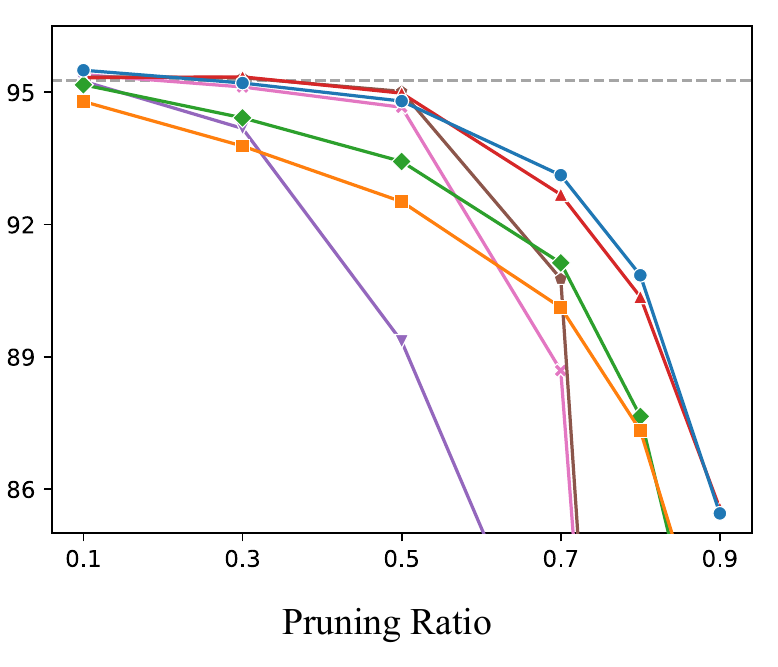}}
\end{minipage}
\begin{minipage}{0.35\linewidth}
  \centering
  \subcaptionbox{CIFAR-100\label{fig: std_cifar100}}{%
    \includegraphics[width=\linewidth]{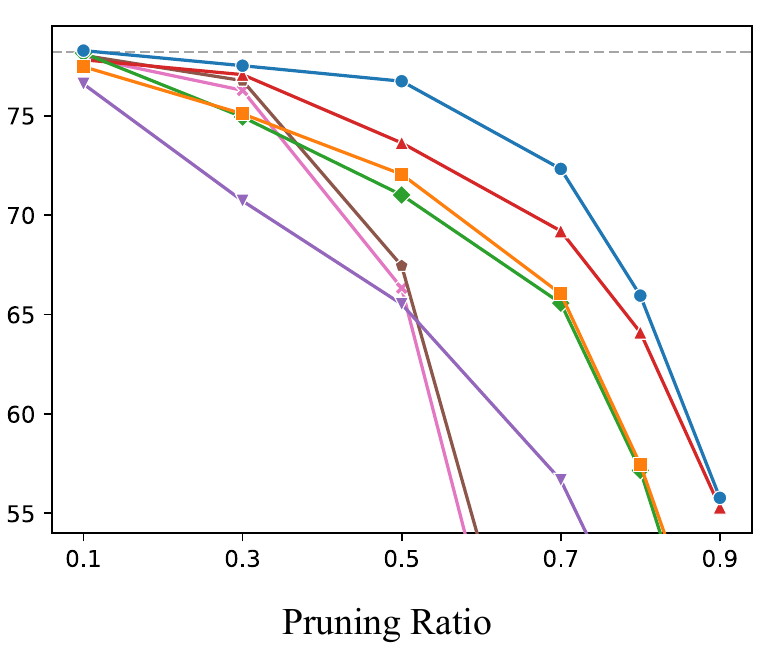}}
\end{minipage}
\vspace{-0.1in}
\caption{Performance comparison between IF-Beta (w/o KD) and other baselines with ResNet-18 on CIFAR-10/100 under standard data pruning setting (\ie, training without KD). The pruning ratio is the fraction of examples removed from the original datasets. The dashed horizontal line denotes the model trained on the full dataset. Detailed numerical results are provided in Appendix~\ref{app:exp}.}
\vspace{-0.3in}
\label{fig:standard_data_pruning}
\end{figure*}

To further verify the generality of IF-Beta beyond KD, we also evaluate it under the standard data pruning setting following CCS~\cite{DBLP:conf/iclr/ccs}.  
In this case, the inner objective in Eq.~\ref{eq: inner obj} is replaced with a standard CE loss, corresponding to conventional supervised training rather than distillation.  
As shown in \cref{fig:standard_data_pruning}, IF-Beta remains competitive and consistently outperforms other baselines across all pruning ratios on CIFAR--10/100.  
These results indicate that IF-Beta is not tied to KD, but can also serve as a general-purpose pruning mechanism under standard supervised learning, without relying on training-dynamics trajectories.

\begin{table}[b]
\centering
\caption{Performance comparison under standard data pruning with various data pruning ratios on CIFAR-10/100.
IF-Beta (w/o KD) uses an in-domain pretrained model to compute influence scores, while IF-Beta* (w/o KD) uses a generic ImageNet-pretrained ResNet-18 (from PyTorch model zoo~\cite{pytorch}). \textbf{Bold} denotes the best results and \underline{underline} denotes the second-best.}
\setlength{\tabcolsep}{4pt}
\scriptsize
\vspace{-0.05in}
\begin{tabular}{lccccccc}
\toprule
Method       & 90\% & 80\% & 70\% & 50\% & 30\% & 10\% \\
\midrule
\multicolumn{7}{c}{CIFAR-10} \\
Random     & 80.17 & 87.65 & 91.13 & 93.43 & 94.42 & 95.17 \\
IF-Beta    & \textbf{85.45} & \textbf{90.85} & \textbf{93.12} & \textbf{94.80} & \textbf{95.21} & \textbf{95.50} \\
IF-Beta*   & \underline{85.34} & \underline{90.22} & \underline{92.22} & \underline{94.59} & \underline{95.14} & \underline{95.37} \\
\midrule
\multicolumn{7}{c}{CIFAR-100} \\
Random     & 44.58 & 57.16 & 65.57 & 71.01 & 74.93 & \underline{78.14} \\
IF-Beta    & \textbf{55.78} & \textbf{65.95} & \textbf{72.32} & \textbf{76.73} & \textbf{77.51} & \textbf{78.27} \\
IF-Beta*   & \underline{53.31} & \underline{63.09} & \underline{69.50} & \underline{74.79} & \underline{76.54} & 77.60 \\
\bottomrule
\end{tabular}
\label{table:transfer_if}
\vspace{-0.15in}
\end{table}


\vspace{0.05in}
\noindent\textbf{Further Discussion.}
IF-Beta is originally designed for KD, where a pretrained teacher model is naturally available to compute influence scores.
However, under standard data pruning, such in-domain pretrained models are typically not provided unless one retrains a model first.
Fortunately, thanks to the nature of FVM~\cite{ye2025robust}, we can simply replace the in-domain model with a generic off-the-shelf pretrained model (\eg, trained on ImageNet) to compute influence scores.
As shown in \cref{table:transfer_if}, under the setting of conventional training without KD, using only an ImageNet-pretrained ResNet-18 from the PyTorch model zoo~\cite{pytorch} to compute IF on CIFAR-10/100 still yields competitive pruning performance, with only minor degradation. 
This means that, for standard data pruning, IF-Beta can completely avoid any retraining cost while still providing strong pruning quality.

\section{Conclusion}
\label{sec:conclusion}

This paper proposes IF-Beta, a data-pruning framework for knowledge distillation that accelerates student training by identifying the most informative samples for distillation.
IF-Beta unifies influence function-based scoring with a learnable Beta Policy 
parameterized by a Beta distribution, which can be efficiently optimized through a bilevel formulation in the teacher's feature space.
The influence function provides a reliable and retraining-free estimate of sample importance, while the Beta Policy overcomes the rigidity of heuristic selection strategies by adaptively learning the optimal sampling distribution.
Extensive experiments on CIFAR-10/100 and ImageNet demonstrate that IF-Beta 
consistently outperforms existing baselines across a wide range of pruning ratios, and remarkably enables students trained on substantially less data and under lower compute budgets to surpass full-data distillation.
Beyond KD, IF-Beta generalizes naturally to standard data pruning settings, demonstrating its effectiveness as a broadly applicable pruning strategy regardless of the training objective.
These results suggest that \textit{which} samples to distill may matter more than \textit{how} to distill them, positioning data curation as a principled and practical lever for improving efficiency in modern training pipelines.

\section*{Acknowledgements}
This work was supported by the National Nature Science Foundation of China (62472097), Shanghai Municipal Science and Technology Commission (Grant No.25511102200), Fudan Kunpeng \& Ascend Center of Cultivation, and Shanghai science and technology committee (Grant No.25511106200). The
computations in this research were performed on the CFFF platform of Fudan University.

%
%
\bibliographystyle{splncs04}
\bibliography{main}

\newpage
\appendix
\onecolumn

\begin{center}
{\Large\bf Supplementary Materials for "Distill on a Diet: Efficient Knowledge Distillation via Learnable Data Pruning"}
\end{center}

\noindent\textbf{Roadmap.}
We first provide implementation and experimental details in Appendix~\ref{app:implementation} and~\ref{app:exp_details}.
Additional experimental results are presented in Appendix~\ref{app:exp}, followed by further analysis on sampling strategies in Appendix~\ref{app:analysis_sampling}.
We then discuss related work in Appendix~\ref{app: related works}, and finally provide the missing proofs and derivations in 
Appendix~\ref{app:ori_PGE} and~\ref{app:derivation}.

\section{Implementation Details}
\label{app:implementation}

\begin{algorithm}[ht]
\small
\caption{Overall Pipeline}
\label{alg:overall}
\begin{algorithmic}[1]
    \Require Pretrained teacher $f_T(\cdot;\theta_T^\star)$; training set $D_{\text{tr}}$; validation set $D_{\text{val}}$; pruning ratio $r$; initial policy parameters $\phi_0$; outer iterations $T$; initial student $f_S(\cdot;\theta_S)$; linear classifier $\bs{C}$
    \vspace{0.4em}
    \hrule
    \vspace{0.4em}
    
    \Statex \textbf{(1) Solving for Flat-Validation-Minima}
    \vspace{0.4em}
    
    \State Initialize $\theta \gets \theta_T^\star$
    
    \State $\tilde{\theta}_T \leftarrow \arg\min_{\theta}\,\hat{R}^\gamma_{\text{val}}(\theta)$ \Comment{Sharpness-Aware Optimization.}

    \vspace{0.4em}
    \Statex \textbf{(2) Computing Influence}
    \vspace{0.4em}
    
    \For{$z_i=(x_i,y_i)\in D_{\mathrm{tr}}$} \Comment{via $\tilde{\theta}$.}
    \State $\mathcal{I}_i \gets \tilde{g}_{z_i}^\top\,\tilde{H}_{\mathrm{val}}^{-1}\,\tilde{g}_{z_i}$
    \EndFor 

    \vspace{0.4em}
    \Statex \textbf{(3) Extracting Frozen Features}
    \vspace{0.4em}
    
    \For{$z_i=(x_i,y_i)\in D_{\mathrm{tr}}$} \Comment{via pretrained $f_T(\cdot;\theta_T^\star)$.}
    \State $\mathbf{f}_i \gets f_T(x_i;\theta_T^\star)$
    \EndFor

    \vspace{0.4em}
    \Statex \textbf{(4) Bilevel Optimization}
    \vspace{0.4em}
    
    \For{$t=1,\ldots,T$} \Comment{Outer-loop.}
    
    \State Sample mask $\bm m \sim \pi_\phi^r := p(\bm m \mid \phi, r)$ \Comment{$\bm{m} = \{m_i\in\{0,1\}\,|\,\|\bm{m}\|_0=(1-r)|D_\text{tr}|\}$.}
    
    \State $\bs{C}^\star(\bm{m}) \gets \underset{\bm C}{\arg \min} \frac{1}{\|\bm m\|_0} \sum_{z_i \in D_{\text{sub}}} \ell_\text{KD}(z_i, \bm C, \theta_T^\star)$ \Comment{Inner-loop.}
    
    \State $\widehat{\mathcal{L}} \gets \mathcal{L}(\bs{C}^\star(\bm{m});\,D_{\text{val}})$
    \State $\phi \gets  \phi - \eta \widehat{\mathcal{L}} \nabla_{\phi}\ln  p(\bs{m}|\phi)$ \Comment{$p(\bm m \mid \phi, r) = \frac{1}{Z_{K_r}} \prod_{i=1}^{N} p_i(\phi)^{m_i} 1_{\|\bm m\|_1 = K_r}(\bm m)$.}
    \EndFor

    \vspace{0.4em}
    \Statex \textbf{(5) Sampling Final Coreset}
    \vspace{0.4em}

    \State $\bm m \sim \pi_\phi^r$ ;\quad $D_{\text{sub}}\gets\{z_i:m_i=1\}$ \Comment{Sample final coreset.}

    \vspace{0.4em}
    \Statex \textbf{(6) KD on the Selected Coreset}
    \vspace{0.4em}
    
    \State $\theta^\star_S \gets \arg\min_{\theta_S}\ \frac{1}{\|\bm m\|_0} \ell_\text{KD}\big(f_S(x_i;\theta_S),\,f_T(x_i;\theta^\star)\big)$ \Comment{$\ell_\text{KD}$ is the standard KD objective.}
    \State \Return $f_S(\cdot;\theta^\star_S)$
\end{algorithmic}
\end{algorithm}

We first present the overall pipeline of IF-Beta for the whole KD scenario in \cref{alg:overall}. 
This high-level algorithm describes the entire procedure from influence computation to policy-driven data pruning and the final student training. 
Next, we provide two concrete implementation clarifications: (i) how we obtain the flat-validation-minima (FVM), and (ii) how we stabilize the bilevel optimization more robustly in practice.

\vspace{0.05in}
\noindent\textbf{Obtaining Flat-Validation-Minima.}
Following Ye et al.~\cite{ye2025robust}, influence functions are sensitive to the flatness of the loss landscape.
In particular, they show that better sharpness-aware optimizers lead to noticeably more reliable influence estimates, with F-SAM \cite{cvpr24fsam} demonstrating the most stable behavior. 
Motivated by this, we adopt F-SAM to solve the flat-validation-minima objective in Eq.~\ref{eq: SAM}.

\begin{algorithm}[t]
\small
\caption{Bilevel Optimization of the Beta Policy $\pi^r_{(\mu,\bar{\tau})}$}
\label{alg:bilevel}
\begin{algorithmic}[1]
\Require normalized scores \(\{\hat{s}_i\}\); frozen features \(\{\mathbf f_i\}\);
pruning ratio \(r\); initial policy parameters \(\mu_0, \tau_0\); step size $t$;
outer steps \(T\)
\vspace{0.4em}\hrule\vspace{0.4em}

\Statex \textbf{Stage A: Discrete optimization of \(\mu\)}
\vspace{0.4em}
\State $\bar{\tau} \gets \tau_0$ 
\For{\(\mu \in \{0,t,2t,\ldots,\lfloor 1/t\rfloor t\}\)}
  \State \(\bm m \sim \pi^r_{(\mu,\bar{\tau})}\) 
  \State \(\bs C^\star(\bm m) \gets \arg\min_{\bs C}\ \frac{1}{\|\bm m\|_0}\sum_i m_i\,\ell_\text{KD}(z_i, \bm C, \theta_T)\)
  \State \(J(\mu) \gets \widehat{\mathcal L}\big(\bs C^\star(\bm m);\,D_{\text{val}}\big)\)
\EndFor
\State \(\mu^\star \gets \arg\min_{\mu} J(\mu)\)

\vspace{0.4em}
\Statex \textbf{Stage B: Continual optimization of \(\tau\)}
\vspace{0.4em}
\State \(\tau \gets \tau_0\)
\For{\(t=1,\ldots,T\)} 
  \State \(\bm m \sim \pi^r_{(\mu,\bar{\tau})}\)
  \State \(\bs C^\star(\bm m) \gets \arg\min_{\bs C}\ \frac{1}{\|\bm m\|_0}\sum_i m_i\,\ell_\text{KD}(z_i, \bm C, \theta_T)\)
  \State \( \hat{g}(\bm m) \gets \widehat{\mathcal L}(\bs C^\star(\bm m);\,D_{\text{val}}) \)
  \State \(\tau \gets \tau - \eta \, \widehat{\mathcal{L}}\left(\theta_S^{*}(\bs{m})\right) \, \nabla_{\phi}\ln  p(\bs{m} | (\mu^\star, \tau), r)\) 
\EndFor

\vspace{0.4em}
\State \Return \((\mu^\star,\,\tau)\)
\end{algorithmic}
\end{algorithm}

\vspace{0.05in}
\noindent\textbf{Bilevel Optimization for the Beta Policy.}
For the Beta Policy, directly optimizing $(\alpha,\beta)$ inside the Beta PDF is unstable in practice. 
Although one could attempt to optimize $\alpha$ and $\beta$ separately, this still does not yield a controllable optimization schedule. 
Therefore, we reparameterize $\phi=(\mu,\tau)$ such that
\begin{equation}
\alpha(\phi)=\tau\mu+1,\qquad 
\beta(\phi)=\tau(1-\mu)+1.
\label{eq:beta_reparam}
\end{equation}
Here, $\mu\in[0,1]$ controls the location of the mode over the normalized influence spectrum, while $\tau>0$ controls the sharpness. 
This reparameterization makes it feasible to decouple the optimization in two stages.

After this reparameterization, we first fix $\tau=\bar{\tau}$ and solve the following bilevel problem only w.r.t.~$\mu$:
\begin{gather}
\min_{0\le \mu \le1}\ \Phi(\mu;\bar{\tau})
= \mathbb{E}_{\bs m \sim \pi_{(\mu,\bar{\tau})}^{\,r}}
   \,\widehat{\mathcal{L}}\left(\bs C^*(\bm m)\right),
\label{eq:mu_outer_proxy}\\
\text{s.t.}\quad 
\bs C^*(\bm m)
= \arg\min_{\bs C}\ 
   \frac{1}{\|\bm m\|_0}\sum_i m_i\,\ell_\text{KD}(z_i, \bm C, \theta_T).
\label{eq:mu_inner_proxy}
\end{gather}
To further simplify this optimization, we discretize the search space of $\mu$. 
Specifically, given a step size $t$, we enumerate $\mu$ over the grid ${0,t,2t,\ldots,\lfloor 1/t\rfloor t}$ and select the one that yields the best validation performance. This converts the continuous search over $\mu$ into a simple discrete selection.

Once $\mu$ is determined, we then fix $\mu=\bar{\mu}$ and optimize only $\tau$ via the same bilevel formulation:
\begin{gather}
\min_{\tau>0}\ \Phi(\bar{\mu};\tau)
= \mathbb{E}_{\bs m \sim \pi_{(\bar{\mu},\tau)}^{\,r}}
   \,\widehat{\mathcal{L}}\left(\bs C^*(\bm m)\right),
\label{eq:tau_outer_proxy}\\
\text{s.t.}\quad
\bs C^*(\bm m)
= \arg\min_{\bs C}\
\frac{1}{\|\bm m\|_0}\sum_i m_i,\ell_\text{KD}(z_i, \bm C, \theta_T).
\label{eq:tau_inner_proxy}
\end{gather}

In summary, this two-stage procedure converts the original two-parameter bilevel optimization into two simpler one-parameter subproblems, which is easier to implement in practice while keeping the underlying objective unchanged. We present the detailed algorithmic procedure for this two-stage bilevel optimization in \cref{alg:bilevel}.

\vspace{0.05in}
\noindent\textbf{Computational Complexity}
The end-to-end cost of obtaining the coreset consists of three components: (i) solving the flat-validation-minima (FVM), (ii) computing influence scores, and (iii) the bilevel policy optimization. First, for solving FVM, we simply fine-tune the pretrained model for 1,000 iterations on the validation set with batch size 128. Second, the influence computation requires a single pass over $D_{\text{val}}$ to obtain $\tilde{g}_{\text{val}}$ and a single pass over $D_{\text{tr}}$ to obtain $\tilde{g}_{\text{tr}}$; the inverse-Hessian is approximated via diagonal Fisher, which adds almost no additional cost beyond these passes. Third, the bilevel stage is also lightweight, because the inner problem \(\bs C^\star\) in Eq.~\ref{eq: easy inner obj} is solved purely in frozen feature space. Empirically, this linear model converges in just $1{\sim}2$ epochs, making the overall bilevel phase minor relative to solving FVM and computing IF.
Putting these together, the total cost of obtaining the coreset from a pretrained teacher (starting from scratch) is well under four full epochs of training.

\section{Experimental Details}
\label{app:exp_details}

\subsection{Details on Baseline}

\vspace{0.05in}
\noindent\textbf{EL2N~\cite{el2n}} 
measures the L2 error between the model prediction and the one-hot label at an early epoch (we use epoch 20 by default). Samples with larger EL2N scores are kept, and samples with smaller scores are pruned.

\vspace{0.05in}
\noindent\textbf{GraNd~\cite{el2n}}
computes the gradient norm of the loss w.r.t. model parameters. Examples with the smaller gradient norm (\ie, easier samples) are removed.

\vspace{0.05in}
\noindent\textbf{Forgetting~\cite{toneva2018forgetting}} counts the number of forgetting events during training (\ie, transitions from correct to incorrect prediction). Samples with fewer forgetting events are considered easier and are pruned.

\vspace{0.05in}
\noindent\textbf{DUAL~\cite{DUAL}} measures the Difficulty and Uncertainty-Aware Lightweight score at an early epoch (we use epoch 60 by default). Consistent with the original implementation, we additionally equip DUAL with its pruning-ratio-adaptive sampling.

\vspace{0.05in}
\noindent\textbf{Medium-Difficulty~\cite{iclr25medium}} ranks samples by the teacher’s cross-entropy loss and selects those in the "medium" range window, discarding very easy and very hard examples.
The variant \textit{Medium-Difficulty*} follows the same selection rule but replaces the standard KD loss in Eq.~\ref{eq:kd_total} with the modified distillation loss proposed in Chen et al.~\cite{iclr25medium}.
For brevity, we refer to Medium-Difficulty as Medium in the appendix.

\vspace{0.05in}
\noindent\textbf{CCS~\cite{DBLP:conf/iclr/ccs}} first applies a hard cutoff to remove the hardest samples and then performs stratified sampling over the remaining score ranges. For the hard-cutoff ratio, we follow the original paper. 
To adapt CCS to KD without relying on training-dynamics trajectories, we replace its original score utilized in Zheng et al.~\cite{DBLP:conf/iclr/ccs} with the IF-based score under the FVM condition, yielding the variant \textit{IF-CCS}.

\vspace{0.05in}
\noindent\textbf{BWS~\cite{bws}} searches over multiple score windows to estimate the optimal sampling interval.
Because the original code is not available and the method requires training a feature extractor for 20 epochs, which is unsuitable for KD, we instead use the teacher as the feature extractor and adopt the same inner-loop objective as Eq.~\ref{eq: easy inner obj} to estimate the optimal window for sampling.
Similar to CCS, we replace its original score utilized in Choi et al.~\cite{bws} with the IF-based score under the FVM condition, yielding the variant \textit{IF-BWS}.

\subsection{Main Experimental Settings}

Unless otherwise specified, all experiments follow the same training hyperparameters across teacher, student, and non-distillation baselines. For KD experiments, \emph{teacher models are always trained on the full training set}, while student models are trained on a pruned subset of the data (with different pruning ratios). For non-KD experiments, the model is trained exactly the same as the student, except that only CE loss is used.

\vspace{0.05in}
\noindent\textbf{CIFAR-10/100.}
For CIFAR-10/100, the student model is always a ResNet-18~\cite{DBLP:conf/cvpr/Resnet}, while we consider four teacher architectures: ResNet-18~\cite{DBLP:conf/cvpr/Resnet}, ResNet-50~\cite{DBLP:conf/cvpr/Resnet}, ResNet-101~\cite{DBLP:conf/cvpr/Resnet}, and WideResNet-28-10~\cite{zagoruyko2016wideresnet}.
Unless otherwise stated, all teachers and students are trained with the same optimization hyperparameters: 40,000 iterations (about 200 epochs for the entire dataset) with batch size 256, using SGD with momentum 0.9, weight decay $2\times10^{-4}$, and an initial learning rate of 0.1 with cosine decay to $10^{-4}$.
Standard data augmentations are used, including $4$-pixel padding with random cropping and random horizontal flipping.
For KD experiments, student models are trained using the loss in Eq.~\ref{eq:kd_total} with $\alpha=0.5$.
Because data augmentation changes every mini-batch, we recompute the teacher logits at every iteration to ensure consistency between teacher predictions and student inputs.
Additionally, we investigate a wide range of pruning ratios, including \{10\%, 30\%, 50\%, 70\%, 80\%, 90\%\}, in all KD scenarios.
For non-KD experiments, the student follows the exact same training schedule but replaces the KD loss with a standard CE loss.
All CIFAR experiments are conducted with random seeds \{0, 1, 2\}.

\vspace{0.05in}
\noindent\textbf{ImageNet.}
For ImageNet, the student model is always ResNet-50, while we consider four teacher architectures: ResNet-50~\cite{DBLP:conf/cvpr/Resnet}, ResNet-101~\cite{DBLP:conf/cvpr/Resnet}, WideResNet-50-2~\cite{zagoruyko2016wideresnet}, and ViT-Base~\cite{DBLP:conf/iclr/ViT}. 
The ResNet-50 teacher is trained from scratch following standard practice, whereas the ResNet-101, WideResNet-50-2, and ViT-Base teachers are initialized from the PyTorch model zoo~\cite{pytorch}.
Concretely, the ResNet-50 teacher is trained for 300,000 iterations (about 60 epochs for the entire dataset) with batch size 256, using SGD with momentum 0.9, weight decay $10^{-4}$, and an initial learning rate of 0.1 with cosine decay to $10^{-4}$.
For data augmentation, we apply the standard ImageNet pipeline consisting of random cropping to $224 \times 224$ and random horizontal flipping.
For KD experiments, student models are trained using the loss in Eq.~\ref{eq:kd_total} with $\alpha=0.5$.
Due to computational constraints, we compute the teacher logits only once at the beginning of training and reuse them throughout the entire KD process.
Student training also uses 300,000 iterations, but with a batch size of 128 (about 30 epochs for the entire dataset).
Additionally, we investigate pruning ratios of \{30\%, 50\%, 70\%, 80\%, 90\%\} in all ImageNet KD experiments.
All ImageNet experiments are conducted with random seeds \{0, 1, 2\}.

\vspace{0.05in}
\noindent\textbf{Hyperparameters for IF-Beta.}
We next summarize the hyperparameters specific to IF-Beta itself, including solving the FVM condition and the bilevel optimization procedure. 
For solving the FVM condition, we finetune the pretrained teacher using F-SAM~\cite{cvpr24fsam} on the validation set for 1,000 iterations (about 2.5 epochs on CIFAR-10/100 training set and about 1 epoch on ImageNet training set) with batch size 128, using SGD with momentum 0.9, weight decay $5\times 10^{-4}$, and an initial learning rate of 0.01 with cosine decay to 0.
During the bilevel optimization, the outer loop runs for 20 iterations with an initial learning rate of 0.1 with cosine decay to 0.01, and the inner loop trains the linear classifier for 1 epoch using the KD-aligned objective in Eq.~\ref{eq: easy inner obj} with a learning rate of $10^{-3}$.

\section{More Results on Experiments}
\label{app:exp}

In this section, we evaluate our proposed IF-Beta through a wide range of analysis. 

\subsection{KD Data Pruning Results}

Here, we provide comprehensive experimental results for knowledge distillation under various data pruning settings. We first evaluate IF-Beta in the homogeneous teacher–student setting, where the teacher and student share the same architecture, which serves as the primary evaluation protocol in prior KD data pruning studies~\cite{baruch2025KDinPruning, iclr25medium}. We then extend our analysis to heterogeneous teachers, considering a broad range of teacher architectures~\cite{DBLP:conf/cvpr/Resnet, zagoruyko2016wideresnet, DBLP:conf/iclr/ViT} for each dataset.

\vspace{0.05in}
\noindent\textbf{Homogeneous Teacher-Student Setting}
As shown in Fig.~\ref{fig: main_kd_full_data}, IF-Beta consistently outperforms all baselines across CIFAR-10, CIFAR-100, and ImageNet in the homogeneous teacher–student setting.
More importantly, IF-Beta is able to exceed the accuracy of full–data distillation while using only a subset of the training set, indicating that a substantial portion of the data is indeed redundant for KD.
Concretely, on CIFAR-10, IF-Beta achieves 95.69\% accuracy using only 70\% of the data—surpassing the full-data KD student (95.50\%).
On CIFAR-100, IF-Beta reaches 79.53\% accuracy even at the 90\% pruning ratio, again slightly surpassing the full-data result (79.38\%).
On ImageNet, under a limited compute budget (300K iterations, approximately 30 epochs for the entire dataset), IF-Beta matches or exceeds full-data distillation: using 50\% of the data already yields 74.02\%, higher than the full-data accuracy (73.54\%), and even with 30\% data, the performance (73.48\%) remains nearly identical.
These results demonstrate that IF-Beta not only provides better pruning decisions than existing baselines, but also enables substantial data reduction without sacrificing—and sometimes even improving—distillation performance.
The detailed numerical results corresponding to Fig.~\ref{fig: main_kd_full_data} are provided in Tab.~\ref{table:cifar10_res18_res18}–\ref{table:imagenet_res50_res50}.

\begin{figure*}[htbp]
\centering

\includegraphics[width=0.88\linewidth]{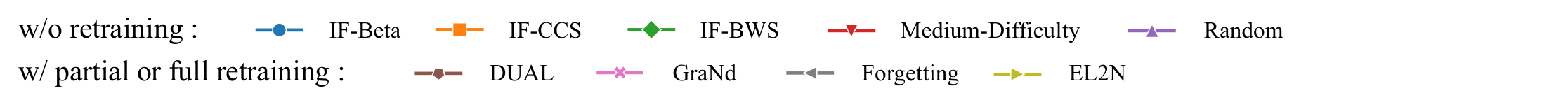}
\\
\vspace{-0.05in}
\raisebox{5pt}{
\begin{minipage}{0.03\linewidth}
  \centering
  \includegraphics[width=\linewidth]{images/main_kd_2.pdf}
\end{minipage}
}
\hspace{-0.1in}
\begin{minipage}{0.28\linewidth}
  \centering
  \subcaptionbox{CIFAR-10}{%
    \includegraphics[width=\linewidth]{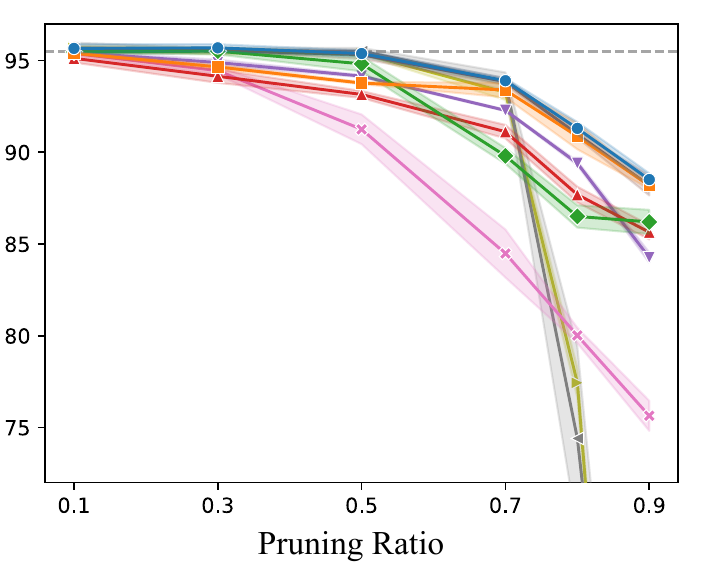}}
\end{minipage}
\hspace{0.05in}
\begin{minipage}{0.28\linewidth}
  \centering
  \subcaptionbox{CIFAR-100}{%
    \includegraphics[width=\linewidth]{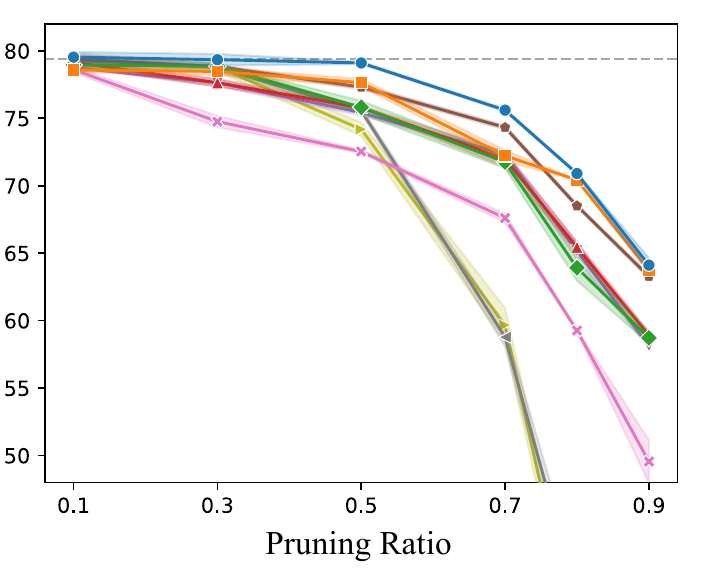}}
\end{minipage}
\hspace{0.05in}
\begin{minipage}{0.28\linewidth}
  \centering
  \subcaptionbox{ImageNet}{%
    \includegraphics[width=\linewidth]{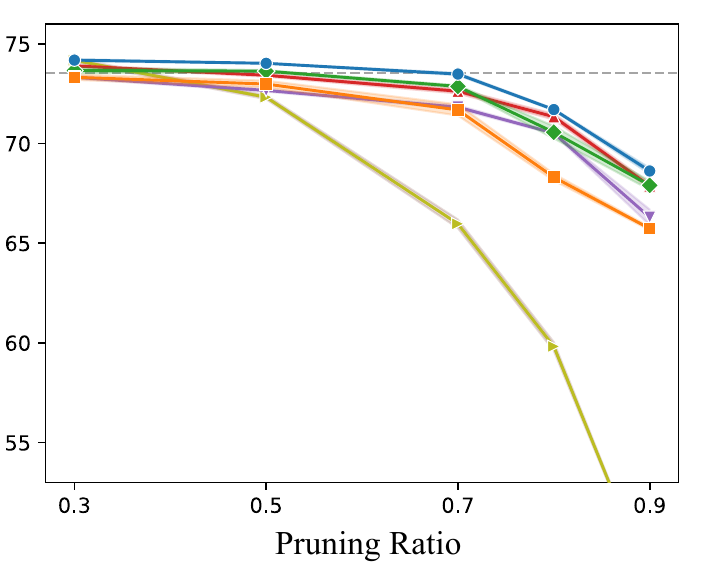}}
\end{minipage}
\vspace{-0.1in}
\caption{Performance comparison between IF-Beta and other baselines on CIFAR-10/100 and ImageNet under the KD setting, where for CIFAR-10/100 both teacher and student are ResNet-18, and for ImageNet both are ResNet-50. The dashed horizontal line denotes the student distilled on the full dataset (without pruning). Detailed numerical results are provided in Tab.~\ref{table:cifar10_res18_res18}-\ref{table:imagenet_res50_res50}.}
\vspace{-0.1in}
\label{fig: main_kd_full_data}
\end{figure*}

\begin{table*}[htbp]
\caption{Test accuracy (\%) with various data pruning ratios on CIFAR-10 under KD setting, where both teacher and student are ResNet-18. Results (mean $\pm$ std) are reported over 3 random runs. \textbf{Bold} denotes the best results and \underline{underline} denotes the second-best.}
\centering
\resizebox{\columnwidth}{!}{
\setlength{\tabcolsep}{2pt}
\scriptsize
\begin{tabular}{l |c c c c c c c}
    \toprule
    \textbf{Method} & \textbf{90\%} & \textbf{80\%} & \textbf{70\%} & \textbf{50\%} & \textbf{30\%} & \textbf{10\%} & \textbf{0\%} \\
    
    \midrule
    &\multicolumn{6}{c}{\textbf{w/ partial or full retraining}} \\
    
    \textbf{EL2N}~\cite{el2n} & 29.08 {\tiny$\pm$0.66} & 77.43 {\tiny$\pm$0.82} & 93.27 {\tiny$\pm$0.42} & \underline{95.38} {\tiny$\pm$0.08} & 95.53 {\tiny$\pm$0.19} & \textbf{95.69} {\tiny$\pm$0.13} & \multirow{11}{*}{95.50{\tiny$\pm$0.08}} \\
    \textbf{Forgetting}~\cite{toneva2018forgetting} & 39.43 {\tiny$\pm$0.30} & 74.40 {\tiny$\pm$4.99} & \underline{93.86} {\tiny$\pm$0.49} & \textbf{95.51} {\tiny$\pm$0.18} & 95.54 {\tiny$\pm$0.16} & \underline{95.66} {\tiny$\pm$0.31} \\
    \textbf{GraNd}~\cite{el2n} & 75.63 {\tiny$\pm$0.82} & 80.01 {\tiny$\pm$0.42} & 84.48 {\tiny$\pm$1.31} & 91.25 {\tiny$\pm$0.81} & 94.45 {\tiny$\pm$0.21} & 95.39 {\tiny$\pm$0.16} \\
    
    \textbf{DUAL}~\cite{DUAL} & \underline{88.19} {\tiny$\pm$0.55} & \underline{90.98} {\tiny$\pm$0.20} & 93.85 {\tiny$\pm$0.19} & 95.32 {\tiny$\pm$0.23} & \underline{95.54} {\tiny$\pm$0.08} & 95.58 {\tiny$\pm$0.17} \\
    
    \cmidrule(lr){1-7}
    &\multicolumn{6}{c}{\textbf{w/o retraining}} \\
    
    \textbf{Random}~\cite{baruch2025KDinPruning} & 84.26 {\tiny$\pm$0.28} & 89.40 {\tiny$\pm$0.11} & 92.28 {\tiny$\pm$0.04} & 94.15 {\tiny$\pm$0.06} & 94.89 {\tiny$\pm$0.16} & 95.44 {\tiny$\pm$0.06} \\
    \textbf{Medium}~\cite{iclr25medium} & 85.64 {\tiny$\pm$0.39} & 87.69 {\tiny$\pm$0.43} & 91.12 {\tiny$\pm$0.37} & 93.15 {\tiny$\pm$0.19} & 94.14 {\tiny$\pm$0.36} & 95.12 {\tiny$\pm$0.23} \\
    \textbf{Medium*}~\cite{iclr25medium} & 85.39 {\tiny$\pm$1.12} & 87.92 {\tiny$\pm$0.06} & 91.05 {\tiny$\pm$0.23} & 92.99 {\tiny$\pm$0.31} & 94.02 {\tiny$\pm$0.30} & 94.97 {\tiny$\pm$0.10} \\
    \textbf{IF-BWS}~\cite{bws} & 86.20 {\tiny$\pm$0.66} & 86.50 {\tiny$\pm$0.61} & 89.80 {\tiny$\pm$0.43} & 94.82 {\tiny$\pm$0.41} & 95.51 {\tiny$\pm$0.18} & 95.47 {\tiny$\pm$0.07} \\
    \textbf{IF-CCS}~\cite{DBLP:conf/iclr/ccs} & 88.18 {\tiny$\pm$0.18} & 90.87 {\tiny$\pm$0.71} & 93.39 {\tiny$\pm$0.40} & 93.77 {\tiny$\pm$0.06} & 94.65 {\tiny$\pm$0.12} & 95.40 {\tiny$\pm$0.07} \\
    \textbf{IF-Beta} (Ours) & \textbf{88.51} {\tiny$\pm$0.34} & \textbf{91.30} {\tiny$\pm$0.35} & \textbf{93.90} {\tiny$\pm$0.14} & \underline{95.38} {\tiny$\pm$0.17} & \textbf{95.69} {\tiny$\pm$0.20} & \underline{95.66} {\tiny$\pm$0.24} \\

    \bottomrule
\end{tabular}
}
\label{table:cifar10_res18_res18}
\end{table*}

\begin{table*}[htbp]
\caption{Test accuracy (\%) with various data pruning ratios on CIFAR-100 under KD setting, where both teacher and student are ResNet-18. Results (mean $\pm$ std) are reported over 3 random runs. \textbf{Bold} denotes the best results and \underline{underline} denotes the second-best.}
\centering
\resizebox{\columnwidth}{!}{
\setlength{\tabcolsep}{2pt}
\scriptsize
\begin{tabular}{l |c c c c c c c}
    \toprule
    \textbf{Method} & \textbf{90\%} & \textbf{80\%} & \textbf{70\%} & \textbf{50\%} & \textbf{30\%} & \textbf{10\%} & \textbf{0\%} \\
    
    \midrule
    &\multicolumn{6}{c}{\textbf{w/ partial or full retraining}} \\
    \textbf{EL2N}~\cite{el2n} & 19.24 {\tiny$\pm$0.67} & 35.64 {\tiny$\pm$1.53} & 59.67 {\tiny$\pm$1.26} & 74.18 {\tiny$\pm$0.50} & 78.80 {\tiny$\pm$0.33} & \underline{79.41} {\tiny$\pm$0.12} & \multirow{11}{*}{79.38{\tiny$\pm$0.12}}  \\
    \textbf{Forgetting}~\cite{toneva2018forgetting} & 29.34 {\tiny$\pm$0.62} & 39.45 {\tiny$\pm$1.82} & 58.77 {\tiny$\pm$0.70} & 75.57 {\tiny$\pm$0.20} & 78.76 {\tiny$\pm$0.33} & 79.32 {\tiny$\pm$0.48} \\
    \textbf{GraNd}~\cite{el2n} & 49.55 {\tiny$\pm$1.60} & 59.26 {\tiny$\pm$0.09} & 67.60 {\tiny$\pm$0.33} & 72.53 {\tiny$\pm$0.12} & 74.75 {\tiny$\pm$0.43} & 78.60 {\tiny$\pm$0.20} \\
    \textbf{DUAL}~\cite{DUAL} & 63.27 {\tiny$\pm$0.14} & 68.51 {\tiny$\pm$0.11} & \underline{74.32} {\tiny$\pm$0.15} & 77.31 {\tiny$\pm$0.17} & 78.83 {\tiny$\pm$0.13} & 79.39 {\tiny$\pm$0.28} \\
    
    \cmidrule(lr){1-7}
    &\multicolumn{6}{c}{\textbf{w/o retraining}} \\
    
    \textbf{Random}~\cite{baruch2025KDinPruning} & 58.13 {\tiny$\pm$0.14} & 65.17 {\tiny$\pm$0.68} & 72.31 {\tiny$\pm$0.14} & 75.42 {\tiny$\pm$0.21} & 77.63 {\tiny$\pm$0.25} & 78.76 {\tiny$\pm$0.02} \\
    \textbf{Medium}~\cite{iclr25medium} & 58.91 {\tiny$\pm$0.27} & 65.45 {\tiny$\pm$0.33} & 71.97 {\tiny$\pm$0.57} & 75.80 {\tiny$\pm$0.10} & 77.63 {\tiny$\pm$0.29} & 79.14 {\tiny$\pm$0.24} \\
    \textbf{Medium*}~\cite{iclr25medium} & 58.59 {\tiny$\pm$0.35} & 65.20 {\tiny$\pm$0.48} & 71.91 {\tiny$\pm$0.94} & 75.76 {\tiny$\pm$0.23} & 77.79 {\tiny$\pm$0.09} & 79.03 {\tiny$\pm$0.35} \\
    \textbf{IF-BWS}~\cite{bws} & 58.71 {\tiny$\pm$0.42} & 63.91 {\tiny$\pm$0.95} & 71.75 {\tiny$\pm$0.46} & 75.81 {\tiny$\pm$0.48} & \underline{78.85} {\tiny$\pm$0.15} & 78.95 {\tiny$\pm$0.27} \\
    \textbf{IF-CCS}~\cite{DBLP:conf/iclr/ccs} & \underline{63.73} {\tiny$\pm$0.14} & \underline{70.42} {\tiny$\pm$0.04} & 72.22 {\tiny$\pm$0.41} & \underline{77.64} {\tiny$\pm$0.28} & 78.47 {\tiny$\pm$0.28} & 78.58 {\tiny$\pm$0.16} \\
    \textbf{IF-Beta} (Ours) & \textbf{64.11} {\tiny$\pm$0.53} & \textbf{70.90} {\tiny$\pm$0.08} & \textbf{75.60} {\tiny$\pm$0.09} & \textbf{79.10} {\tiny$\pm$0.09} & \textbf{79.34} {\tiny$\pm$0.43} & \textbf{79.53} {\tiny$\pm$0.38} \\

    \bottomrule
\end{tabular}
}
\label{table:cifar100_res18_res18}
\end{table*}

\begin{table*}[htbp]
\caption{Test accuracy (\%) with various data pruning ratios on ImageNet under KD setting, where both teacher and student are ResNet-50. Results (mean $\pm$ std) are reported over 3 random runs. \textbf{Bold} denotes the best results and \underline{underline} denotes the second-best.}
\centering
{
\setlength{\tabcolsep}{2pt}
\scriptsize
\begin{tabular}{l |c c c c c c}
    \toprule
    \textbf{Method} & \textbf{90\%} & \textbf{80\%} & \textbf{70\%} & \textbf{50\%} & \textbf{30\%} & \textbf{0\%} \\
    \midrule
    &\multicolumn{5}{c}{\textbf{w/ partial or full retraining}} \\
    \textbf{EL2N}~\cite{el2n} & 47.99 {\tiny$\pm$0.07} & 59.82 {\tiny$\pm$0.22} & 65.97 {\tiny$\pm$0.22} & 72.32 {\tiny$\pm$0.10} & \underline{74.17} {\tiny$\pm$0.09} & \multirow{7}{*}{73.54{\tiny$\pm$0.12}} \\
    \cmidrule(lr){1-6}
    &\multicolumn{5}{c}{\textbf{w/o retraining}} \\
    \textbf{Random}~\cite{baruch2025KDinPruning} & 66.30 {\tiny$\pm$0.39} & 70.52 {\tiny$\pm$0.05} & 71.82 {\tiny$\pm$0.09} & 72.66 {\tiny$\pm$0.06} & 73.32 {\tiny$\pm$0.07} & \\
    \textbf{Medium}~\cite{iclr25medium} & 67.87 {\tiny$\pm$0.13} & \underline{71.34} {\tiny$\pm$0.14} & 72.62 {\tiny$\pm$0.11} & 73.42 {\tiny$\pm$0.07} & 73.89 {\tiny$\pm$0.06} & \\
    \textbf{Medium*}~\cite{iclr25medium} & 67.84 {\tiny$\pm$0.14} & 71.06 {\tiny$\pm$0.12} & 72.79 {\tiny$\pm$0.17} & 73.48 {\tiny$\pm$0.08} & 74.01 {\tiny$\pm$0.07} & \\
    \textbf{IF-BWS}~\cite{bws} & \underline{67.90} {\tiny$\pm$0.22} & 70.56 {\tiny$\pm$0.31} & \underline{72.86} {\tiny$\pm$0.06} & \underline{73.63} {\tiny$\pm$0.05} & 73.66 {\tiny$\pm$0.15} & \\
    \textbf{IF-CCS}~\cite{DBLP:conf/iclr/ccs} & 65.73 {\tiny$\pm$0.08} & 68.31 {\tiny$\pm$0.17} & 71.68 {\tiny$\pm$0.28} & 72.98 {\tiny$\pm$0.17} & 73.32 {\tiny$\pm$0.13} & \\
    \textbf{IF-Beta} (Ours) & \textbf{68.62} {\tiny$\pm$0.15} & \textbf{71.70} {\tiny$\pm$0.04} & \textbf{73.48} {\tiny$\pm$0.06} & \textbf{74.02} {\tiny$\pm$0.01} & \textbf{74.18} {\tiny$\pm$0.05} & \\
    \bottomrule
\end{tabular}
}
\label{table:imagenet_res50_res50}
\end{table*}

\vspace{0.05in}
\noindent\textbf{Heterogeneous Teachers}
As shown in Figs.~\ref{fig: cifar10_teachers}–\ref{fig: imagenet_teachers}, we further evaluate IF-Beta under heterogeneous teacher architectures.
Since Random, IF-CCS, and IF-BWS are the most competitive non-retraining baselines in the homogeneous setting, we focus on comparing against these three methods here.
Across all datasets and teacher architectures, including ResNet-50/101~\cite{DBLP:conf/cvpr/Resnet}, WideResNet-28-10/50-2~\cite{zagoruyko2016wideresnet}, and ViT-Base~\cite{DBLP:conf/iclr/ViT}, IF-Beta consistently outperforms all baselines at every pruning ratio.
Notably, IF-Beta remains stable even when the teacher becomes substantially stronger than the student, demonstrating that our Beta-parameterized sampling policy adapts well to different levels of teacher capacity.
The corresponding numerical results are summarized in Tables~\ref{table: cifar10_dif_teacher}–\ref{table: imagenet_dif_teacher}.
\footnote{Due to limited computational resources, all experiments in this subsection are conducted with a single random seed (seed = 0), though we observe that the overall trend is highly consistent with the multi-seed results reported in the homogeneous setting.}

\begin{figure}[htbp]
    \centering
    \includegraphics[width=\columnwidth]{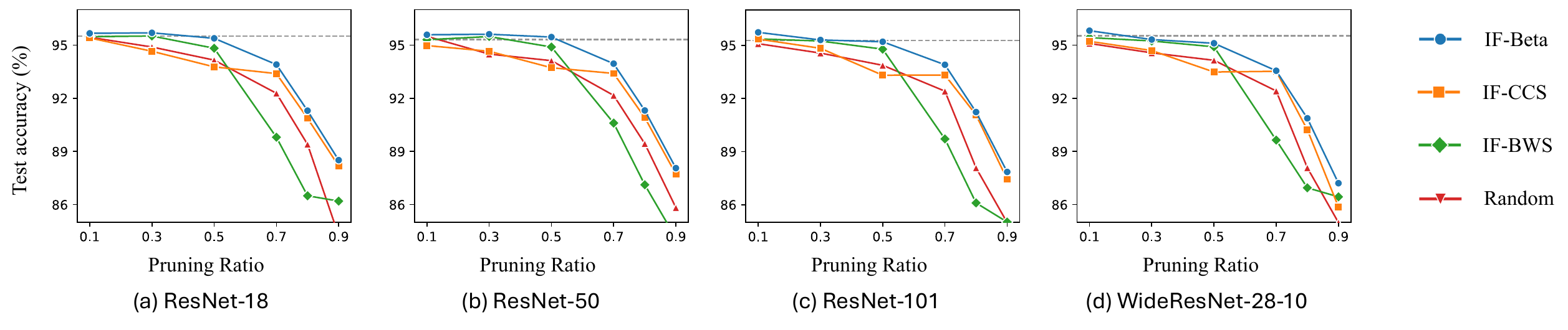}
    \caption{
        Performance comparison between IF-Beta and other data-pruning baselines on CIFAR-10 under the KD setting.
        Across all experiments, the student network is fixed as ResNet-18, while the teacher varies among ResNet-50, ResNet-101, and WideResNet-28-10.
        The dashed horizontal line denotes the student distilled on the full dataset (without pruning).
        Detailed numerical results are provided in \cref{table:cifar10_res18_res18} and \cref{table: cifar10_dif_teacher}.
    }
    \label{fig: cifar10_teachers}
\end{figure}

\begin{figure}[htbp]
    \centering
    \includegraphics[width=\columnwidth]{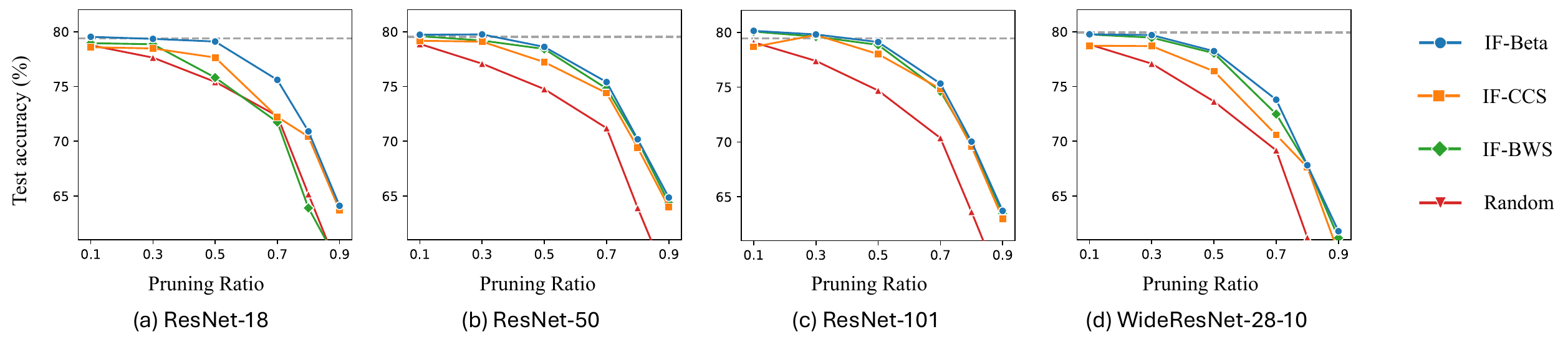}
    \caption{
        Performance comparison between IF-Beta and other data-pruning baselines on CIFAR-100 under the KD setting.
        Across all experiments, the student network is fixed as ResNet-18, while the teacher varies among ResNet-50, ResNet-101, and WideResNet-28-10.
        The dashed horizontal line denotes the student distilled on the full dataset (without pruning).
        Detailed numerical results are provided in \cref{table:cifar100_res18_res18} and \cref{table: cifar100_dif_teacher}.
    }
    \label{fig: cifar100_teachers}
\end{figure}

\begin{figure}[htbp]
    \centering
    \includegraphics[width=\columnwidth]{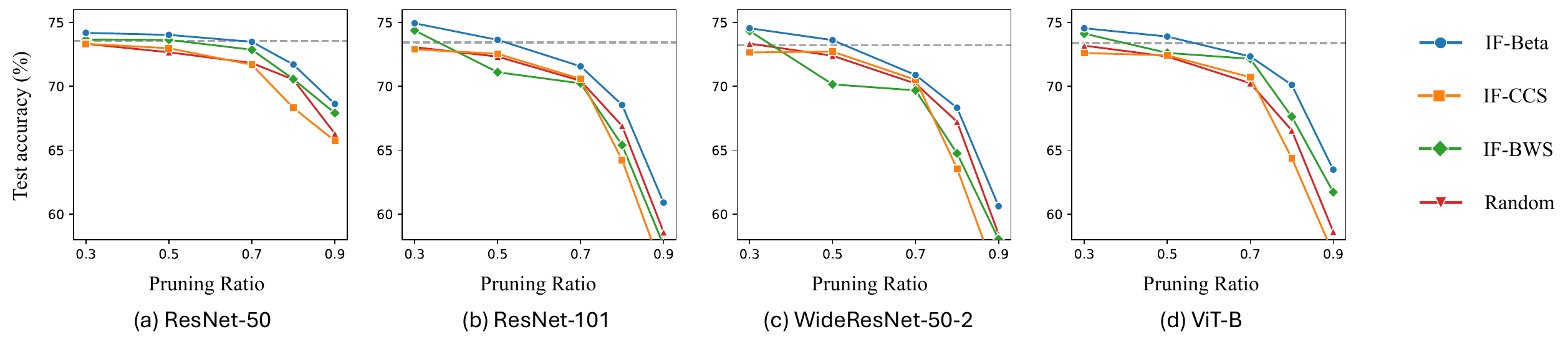}
    \caption{
        Performance comparison between IF-Beta and other data-pruning baselines on ImageNet under the KD setting.
        Across all experiments, the student network is fixed as ResNet-50, while the teacher varies among ResNet-101, WideResNet-50-2, and ViT-Base.
        The dashed horizontal line denotes the student distilled on the full dataset (without pruning).
        Detailed numerical results are provided in \cref{table:imagenet_res50_res50} and \cref{table: imagenet_dif_teacher}.
    }
    \label{fig: imagenet_teachers}
\end{figure}

\begin{table*}[htbp]
\caption{Test accuracy (\%) on CIFAR-10 under the KD setting with various data pruning ratios and different teacher models.
The student network is consistently ResNet-18.
\textbf{Bold} denotes the best results and \underline{underline} denotes the second-best.}
\centering
\setlength{\tabcolsep}{2pt}
\scriptsize
\begin{tabular}{c l |c c c c c c c}
    \toprule
    \textbf{Teacher} & \textbf{Method} & \textbf{90\%} & \textbf{80\%} & \textbf{70\%} & \textbf{50\%} & \textbf{30\%} & \textbf{10\%} & \textbf{0\%} \\
    \midrule
    
    \multirow{4}{*}{ResNet-50} 
    & \textbf{Random}~\cite{baruch2025KDinPruning} & 85.83 & 89.44 & 92.16 & 94.12 & 94.49 & \underline{95.49}  & \multirow{4}{*}{95.31} \\ 
    & \textbf{IF-CCS}~\cite{DBLP:conf/iclr/ccs} & \underline{87.72} & \underline{90.91} & \underline{93.40} & 93.72 & 94.65 & 94.97 & \\
    & \textbf{IF-BWS}~\cite{bws} & 84.11 & 87.12 & 90.60 & \underline{94.89} & \underline{95.48} & 95.30 & \\
    & \textbf{IF-Beta} (Ours) & \textbf{88.06} & \textbf{91.31} & \textbf{93.95} & \textbf{95.45}& \textbf{95.61} & \textbf{95.58} &  \\
    
    \midrule
    
    \multirow{4}{*}{ResNet-101}
    & \textbf{Random}~\cite{baruch2025KDinPruning} & 84.99 & 88.08 & 92.41 & 93.87 & 94.56 & 95.09 & \multirow{4}{*}{95.27} \\
    & \textbf{IF-CCS}~\cite{baruch2025KDinPruning} & \underline{87.43} & \underline{91.07} & \underline{93.32} & 93.31 & 94.83 & 95.35 & \\
    & \textbf{IF-BWS}~\cite{bws} & 85.02 & 86.10 & 89.71 & \underline{94.78} & \underline{95.24} & \underline{95.36} & \\
    & \textbf{IF-Beta} (Ours) & \textbf{87.85} & \textbf{91.23} & \textbf{93.90} & \textbf{95.20} & \textbf{95.30} & \textbf{95.73} & \\
    
    \midrule
    
    \multirow{4}{*}{WideResNet-28-10}
    & \textbf{Random}~\cite{baruch2025KDinPruning} & 84.99 & \underline{88.08} & 92.41 & 94.14 & 94.56 & 95.09  & \multirow{4}{*}{95.52} \\ 
    & \textbf{IF-CCS}~\cite{baruch2025KDinPruning} & 85.87 & 90.22 & \underline{93.52} & 93.48 & 94.69 & 95.21 & \\
    & \textbf{IF-BWS}~\cite{bws} & \underline{86.44} & 86.95 & 89.64 & \underline{94.90} & \underline{95.23} & \underline{95.42} & \\
    & \textbf{IF-Beta} (Ours) & \textbf{87.21} & \textbf{90.87} & \textbf{93.55} & \textbf{95.10} & \textbf{95.31} & \textbf{95.81} & \\
    
    \bottomrule
\end{tabular}
\label{table: cifar10_dif_teacher}
\end{table*}

\begin{table*}[htbp]
\caption{Test accuracy (\%) on CIFAR-100 under the KD setting with various data pruning ratios and different teacher models. 
The student network is consistently ResNet-18.
\textbf{Bold} denotes the best results and \underline{underline} denotes the second-best.}
\centering
\setlength{\tabcolsep}{2pt}
\scriptsize
\begin{tabular}{c l |c c c c c c c}
    \toprule
    \textbf{Teacher} & \textbf{Method} & \textbf{90\%} & \textbf{80\%} & \textbf{70\%} & \textbf{50\%} & \textbf{30\%} & \textbf{10\%} & \textbf{0\%} \\
    \midrule
    
    \multirow{4}{*}{ResNet-50} 
    & \textbf{Random}~\cite{baruch2025KDinPruning} & 57.07 & 63.92 & 71.20 & 74.76 & 77.08 & 78.88  & \multirow{4}{*}{79.53} \\
    & \textbf{IF-CCS}~\cite{baruch2025KDinPruning} & 63.99 & 69.41 & 74.40 & 77.22 & 79.09 & 79.17 & \\
    & \textbf{IF-BWS}~\cite{bws} & \underline{64.36} & \underline{70.07} & \underline{74.81} & \underline{78.42} & \underline{79.19} & \underline{79.59} & \\
    & \textbf{IF-Beta} (Ours) & \textbf{64.86} & \textbf{70.18} & \textbf{75.40} & \textbf{78.61} & \textbf{79.75} & \textbf{79.72} & \\
    
    \midrule
    
    \multirow{4}{*}{ResNet-101}
    & \textbf{Random}~\cite{baruch2025KDinPruning} & 56.66 & 63.63 & 70.36 & 74.69 & 77.39 & 79.08 & \multirow{4}{*}{79.44} \\
    & \textbf{IF-CCS}~\cite{baruch2025KDinPruning} & 63.00 & 69.59 & \underline{74.84} & 78.01 & \underline{79.75} & 78.67 & \\
    & \textbf{IF-BWS}~\cite{bws} & \underline{63.48} & \underline{69.81} & 74.62 & \underline{78.85} & 79.60 & \underline{80.06} & \\
    & \textbf{IF-Beta} (Ours) & \textbf{63.72} & \textbf{70.03} & \textbf{75.32} & \textbf{79.11} & \textbf{79.80} & \textbf{80.13} & \\
    
    \midrule
    
    \multirow{4}{*}{WideResNet-28-10}
    & \textbf{Random}~\cite{baruch2025KDinPruning} & 52.67 & 61.22 & 69.17 & 73.63 & 77.08 & 78.83  & \multirow{4}{*}{79.93} \\
    & \textbf{IF-CCS}~\cite{baruch2025KDinPruning} & 59.81 & \underline{67.62} & 70.58 & 76.39 & 78.68 & 78.72 & \\
    & \textbf{IF-BWS}~\cite{bws} & \underline{61.18} & \textbf{67.81} & \underline{72.48} & \underline{78.03} & \underline{79.46} & \underline{79.74} & \\
    & \textbf{IF-Beta} (Ours) & \textbf{61.79} & \textbf{67.81} & \textbf{73.79} & \textbf{78.22} & \textbf{79.69} & \textbf{79.76} & \\
    
    \bottomrule
\end{tabular}
\label{table: cifar100_dif_teacher}
\end{table*}

\begin{table*}[htbp]
\caption{Test accuracy (\%) on ImageNet under the KD setting with various data pruning ratios and different teacher models. 
The student network is consistently ResNet-50.
\textbf{Bold} denotes the best results and \underline{underline} denotes the second-best.}
\centering
\setlength{\tabcolsep}{2pt}
\scriptsize
\begin{tabular}{c l |c c c c c c c}
    \toprule
    \textbf{Teacher} & \textbf{Method} & \textbf{90\%} & \textbf{80\%} & \textbf{70\%} & \textbf{50\%} & \textbf{30\%} & \textbf{0\%} \\
    \midrule

    \multirow{4}{*}{ResNet-101}
    & \textbf{Random}~\cite{baruch2025KDinPruning} & \underline{58.59} & \underline{66.90} & 70.43 & 72.32 & 73.06 & \multirow{4}{*}{73.43} \\
    & \textbf{IF-CCS}~\cite{baruch2025KDinPruning} & 55.41 & 64.23 & \underline{70.58} & \underline{72.54} & 72.89 & \\
    & \textbf{IF-BWS}~\cite{bws} & 57.56 & 65.40 & 70.23 & 71.10 & \underline{74.37} & \\
    & \textbf{IF-Beta} (Ours) & \textbf{60.90} & \textbf{68.54} & \textbf{71.56} & \textbf{73.63} & \textbf{74.93} & \\
    
    \midrule

    \multirow{4}{*}{WideResNet-50-2} 
    & \textbf{Random}~\cite{baruch2025KDinPruning} & \underline{58.43} & \underline{67.23} & 70.21 & 72.39 & 73.34 & \multirow{4}{*}{73.21} \\
    & \textbf{IF-CCS}~\cite{baruch2025KDinPruning} & 54.78 & 63.55 & \underline{70.51} & \underline{72.71} & 72.65 & \\
    & \textbf{IF-BWS}~\cite{bws} & 58.04 & 64.75 & 69.68 & 70.15 & \underline{74.33} & \\
    & \textbf{IF-Beta} (Ours) & \textbf{60.62} & \textbf{68.32} & \textbf{70.88} & \textbf{73.61} & \textbf{74.54} & \\
    
    \midrule
    
    \multirow{4}{*}{ViT-Base}
    & \textbf{Random}~\cite{baruch2025KDinPruning} & 58.63 & 66.55 & 70.24 & 72.33 & 73.18 & \multirow{4}{*}{73.38} \\
    & \textbf{IF-CCS}~\cite{baruch2025KDinPruning} & 56.91 & 64.38 & 70.71 & 72.42 & 72.60 & \\
    & \textbf{IF-BWS}~\cite{bws} & \underline{61.72} & \underline{67.62} & \underline{72.15} & \underline{72.61} & \underline{74.12} & \\
    & \textbf{IF-Beta} (Ours) & \textbf{63.48} & \textbf{70.11} & \textbf{72.33} & \textbf{73.89} & \textbf{74.54} & \\
    
    \bottomrule
\end{tabular}
\label{table: imagenet_dif_teacher}
\end{table*}

\subsection{Comparison with Other Efficient KD Approaches}
\label{app:comparison_other_kd}

We further compare IF-Beta with recently proposed efficient KD techniques that aim to accelerate training without relying on a full teacher model during optimization. Following the evaluation protocol of Chen et al.~\cite{iclr25medium}, we consider distilling a ResNet-34 teacher~\cite{DBLP:conf/cvpr/Resnet} into a MobileNet student~\cite{mobilenet} on ImageNet with computing on a single NVIDIA A100 GPU.

In their setting, the Baseline corresponds to standard supervised training of MobileNet on the full ImageNet training set for 100 epochs without KD~\cite{mobilenet}. The self-knowledge distillation methods Zipf's LS~\cite{DBLP:conf/eccv/LiangLBZTLF22} and USKD~\cite{DBLP:conf/iccv/YangZLZY023} are also trained for 100 epochs on the full dataset, but with additional self-distillation losses on top of the baseline objective, which slightly increases their computational cost. The test accuracies and training times of these three methods and Medium-Difficulty* are directly cited from Chen et al.~\cite{iclr25medium}.

For a fair comparison, we adopt the same teacher–student pair (ResNet-34 $\to$ MobileNet) and the same 70\% retained subset as Chen et al.~\cite{iclr25medium}. Since the MobileNet student is considerably smaller, we train with batch size of 256, identical to their setup. Consistent with our main ImageNet experiments, we keep the total training budget at 300K iterations (approximately 85 epochs over the 70\% subset) and compute teacher logits only once at the beginning of training, reusing them throughout the KD process. Our student models are trained using the simplest KD loss in Eq.~\ref{eq:kd_total} with $\alpha=0.3$, while all other hyperparameters follow Appendix~\ref{app:exp_details}. 

As summarized in Table~\ref{table:skd_compare}, IF-Beta achieves the best performance of 71.20\%, while simultaneously reducing the overall training time, which requires only 19.24 hours, less than half of the Baseline cost (39.86 hours) and almost half of Medium-Difficulty*’s cost (36.98 hours).
There are several contributing factors:
(i) \textit{reduced data volume} — IF-Beta operates on a subset (pruned 30\% data), directly reducing data volume;
(ii) \textit{fewer effective epochs} — our training schedule corresponds to about 85 epochs, compared to 100 epochs in other baselines;
(iii) \textit{no repeated teacher queries} — unlike Medium-Difficulty*, IF-Beta does not require querying the teacher every epoch, which is particularly costly because the teacher (ResNet-34) is substantially larger than the MobileNet student; and
(iv) \textit{higher-quality subset} — the pruned subset selected by IF-Beta contains more informative samples, which preserves the effectiveness of distillation even under one-shot teacher logits. 
This mitigates the absence of per-iteration teacher guidance and enables IF-Beta to attain higher final accuracy.
Overall, even after including the cost of performing IF-Beta data pruning, our total training remains substantially lower than both baseline and Medium-Difficulty*, while achieving superior test accuracy.
This demonstrates the importance of selecting an informative subset for KD and highlights IF-Beta as a more effective and more efficient alternative to existing efficient KD methods.

\subsection{Standard (Non-KD) Data Pruning Results}

Here, we provide comprehensive experimental results for standard data pruning settings (training without KD), exactly following the prior settings~\cite{DBLP:conf/iclr/ccs, bws}.
As shown in Fig.~\ref{fig:standard_data_pruning_detail}, IF-Beta also demonstrates strong performance under the standard (non-KD) data pruning setting. 
Across both CIFAR-10 and CIFAR-100, IF-Beta consistently surpasses other baselines under almost all pruning ratios. 
Notably, IF-Beta is able to even exceed the accuracy of the full-data model with less data. 
The detailed numerical results corresponding to Fig.~\ref{fig:standard_data_pruning_detail} are provided in Tables~\ref{table:cifar10_res18}–\ref{table:cifar100_res18}.

\begin{figure}[htbp]
\begin{minipage}{0.98\columnwidth}
    \hfill  
    \includegraphics[width=0.88\columnwidth]{images/main_kd_1.pdf}
\end{minipage}
\\
\vspace{-0.05in}
\centering
\raisebox{6pt}{
\begin{minipage}{0.03\columnwidth}
  \centering
  \includegraphics[width=\linewidth]{images/std_data_pruning_2.pdf}
\end{minipage}
}
\hspace{-0.05in}
\begin{minipage}{0.28\columnwidth}
  \centering
  \subcaptionbox{CIFAR-10}{%
    \includegraphics[width=\linewidth]{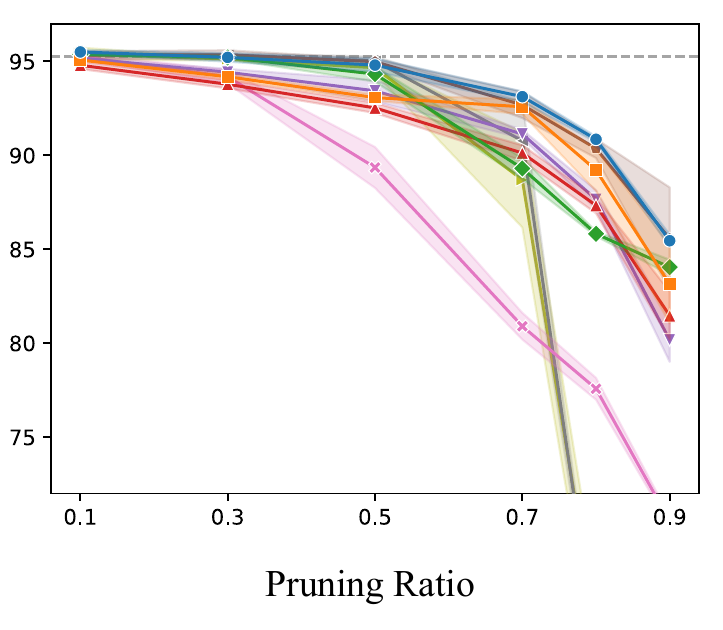}}
\end{minipage}
\hspace{0.2in}
\begin{minipage}{0.28\columnwidth}
  \centering
  \subcaptionbox{CIFAR-100}{%
    \includegraphics[width=\linewidth]{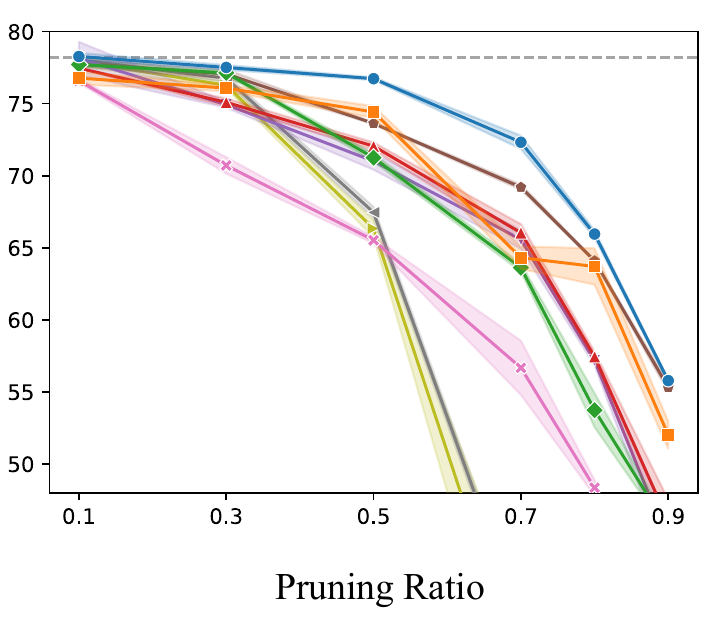}}
\end{minipage}
\caption{Performance comparison between IF-Beta (w/o KD) and other baselines with ResNet-18 on CIFAR-10/100 under standard data pruning setting (\ie, training without KD). The dashed horizontal line denotes the model trained on the full dataset. Detailed numerical results are provided in Tab.~\ref{table:cifar10_res18}-\ref{table:cifar100_res18}.}
\label{fig:standard_data_pruning_detail}
\end{figure}

\begin{table*}[htbp]
\caption{Test accuracy (\%) with ResNet-18 on CIFAR-10 under standard data pruning setting (\ie, training without KD). 
 Results (mean $\pm$ std) are reported over 3 random runs.
\textbf{Bold} denotes the best results and \underline{underline} denotes the second-best.}
\centering
\resizebox{\columnwidth}{!}{
\setlength{\tabcolsep}{2pt}
\scriptsize
\begin{tabular}{l |c c c c c c c}
    \toprule
    \textbf{Method} & \textbf{90\%} & \textbf{80\%} & \textbf{70\%} & \textbf{50\%} & \textbf{30\%} & \textbf{10\%} & \textbf{0\%} \\
    \midrule
    &\multicolumn{6}{c}{\textbf{w/ partial or full retraining}} \\

    \textbf{EL2N}~\cite{el2n} & 22.41 {\tiny$\pm$1.25} & 64.94 {\tiny$\pm$2.19} & 88.69 {\tiny$\pm$2.53} & 94.66 {\tiny$\pm$0.32} & 95.12 {\tiny$\pm$0.07} & \underline{95.40} {\tiny$\pm$0.33} & \multirow{10}{*}{95.26{\tiny$\pm$0.11}} \\
    \textbf{Forgetting}~\cite{toneva2018forgetting} & 35.08 {\tiny$\pm$1.18} & 63.76 {\tiny$\pm$0.42} & 90.77 {\tiny$\pm$2.07} & \textbf{95.02} {\tiny$\pm$0.24} & \underline{95.33} {\tiny$\pm$0.10} & 95.35 {\tiny$\pm$0.09} \\
    \textbf{GraNd}~\cite{el2n} & 70.52 {\tiny$\pm$0.24} & 77.57 {\tiny$\pm$0.58} & 80.90 {\tiny$\pm$0.71} & 89.35 {\tiny$\pm$1.09} & 94.18 {\tiny$\pm$0.36} & 95.25 {\tiny$\pm$0.16} \\
    
    \textbf{DUAL}~\cite{DUAL} & \textbf{85.56} {\tiny$\pm$2.73} & \underline{90.36} {\tiny$\pm$0.51} & \underline{92.68} {\tiny$\pm$0.70} & \underline{94.98} {\tiny$\pm$0.14} & \textbf{95.35} {\tiny$\pm$0.24} & 95.33 {\tiny$\pm$0.21} \\
    
    \cmidrule(lr){1-7}
    &\multicolumn{6}{c}{\textbf{w/o retraining}} \\
    
    \textbf{Random}~\cite{baruch2025KDinPruning} & 80.17 {\tiny$\pm$1.17} & 87.65 {\tiny$\pm$0.50} & 91.13 {\tiny$\pm$0.22} & 93.43 {\tiny$\pm$0.57} & 94.42 {\tiny$\pm$0.18} & 95.17 {\tiny$\pm$0.03} \\
    \textbf{Medium-Difficulty}~\cite{iclr25medium} & 81.46 {\tiny$\pm$1.29} & 87.33 {\tiny$\pm$0.42} & 90.12 {\tiny$\pm$0.41} & 92.52 {\tiny$\pm$0.27} & 93.78 {\tiny$\pm$0.22} & 94.79 {\tiny$\pm$0.20} \\
    
    \textbf{IF-BWS}~\cite{bws} & 84.04 {\tiny$\pm$0.40} & 85.82 {\tiny$\pm$0.13} & 89.30 {\tiny$\pm$0.55} & 94.32 {\tiny$\pm$0.40} & 95.21 {\tiny$\pm$0.09} & 95.37 {\tiny$\pm$0.06} \\
    \textbf{IF-CCS}~\cite{DBLP:conf/iclr/ccs} & 83.14 {\tiny$\pm$2.09} & 89.21 {\tiny$\pm$1.12} & 92.58 {\tiny$\pm$0.39} & 93.07 {\tiny$\pm$0.17} & 94.18 {\tiny$\pm$0.15} & 95.07 {\tiny$\pm$0.02} \\
    \textbf{IF-Beta} (Ours) & \underline{85.45} {\tiny$\pm$0.33} & \textbf{90.85} {\tiny$\pm$0.15} & \textbf{93.12} {\tiny$\pm$0.28} & 94.80 {\tiny$\pm$0.35} & 95.21 {\tiny$\pm$0.23} & \textbf{95.50} {\tiny$\pm$0.11} \\
    \bottomrule
\end{tabular}
}
\label{table:cifar10_res18}
\end{table*}

\begin{table*}[tb]
\caption{Test accuracy (\%) with ResNet-18 on CIFAR-100 under standard data pruning setting (\ie, training without KD). 
 Results (mean $\pm$ std) are reported over 3 random runs.
\textbf{Bold} denotes the best results and \underline{underline} denotes the second-best.}
\centering
\resizebox{\columnwidth}{!}{
\setlength{\tabcolsep}{2pt}
\scriptsize
\begin{tabular}{l |c c c c c c c}
    \toprule
    \textbf{Method} & \textbf{90\%} & \textbf{80\%} & \textbf{70\%} & \textbf{50\%} & \textbf{30\%} & \textbf{10\%} & \textbf{0\%} \\
    
    \midrule
    &\multicolumn{6}{c}{\textbf{w/ partial or full retraining}} \\

    \textbf{EL2N}~\cite{el2n} & 8.14 {\tiny$\pm$0.32} & 15.56 {\tiny$\pm$0.71} & 35.33 {\tiny$\pm$4.94} & 66.33 {\tiny$\pm$0.60} & 76.27 {\tiny$\pm$0.23} & 77.95 {\tiny$\pm$0.35} & \multirow{10}{*}{78.20{\tiny$\pm$0.30}} \\
    \textbf{Forgetting}~\cite{toneva2018forgetting} & 16.14 {\tiny$\pm$0.98} & 24.29 {\tiny$\pm$0.30} & 39.06 {\tiny$\pm$0.57} & 67.45 {\tiny$\pm$0.39} & 76.76 {\tiny$\pm$0.24} & 78.03 {\tiny$\pm$0.31} \\
    \textbf{GraNd}~\cite{el2n} & 34.90 {\tiny$\pm$1.18} & 48.35 {\tiny$\pm$0.54} & 56.67 {\tiny$\pm$1.88} & 65.52 {\tiny$\pm$0.22} & 70.71 {\tiny$\pm$0.57} & 76.59 {\tiny$\pm$0.14} \\
    \textbf{DUAL}~\cite{DUAL} & \underline{55.31} {\tiny$\pm$0.32} & \underline{64.10} {\tiny$\pm$0.07} & \underline{69.20} {\tiny$\pm$0.18} & 73.63 {\tiny$\pm$0.07} & 77.06 {\tiny$\pm$0.31} & 77.81 {\tiny$\pm$0.27} \\
    
    \cmidrule(lr){1-7}
    &\multicolumn{6}{c}{\textbf{w/o retraining}} \\
    
    \textbf{Random}~\cite{baruch2025KDinPruning} & 44.58 {\tiny$\pm$1.53} & 57.16 {\tiny$\pm$0.33} & 65.57 {\tiny$\pm$0.64} & 71.01 {\tiny$\pm$0.60} & 74.93 {\tiny$\pm$0.21} & \underline{78.14} {\tiny$\pm$1.15} \\
    \textbf{Medium-Difficulty}~\cite{iclr25medium} & 46.06 {\tiny$\pm$1.47} & 57.45 {\tiny$\pm$0.35} & 66.04 {\tiny$\pm$0.60} & 72.05 {\tiny$\pm$0.26} & 75.09 {\tiny$\pm$0.27} & 77.46 {\tiny$\pm$0.02} \\
    \textbf{IF-BWS}~\cite{bws} & 45.53 {\tiny$\pm$0.33} & 53.73 {\tiny$\pm$1.18} & 63.62 {\tiny$\pm$0.28} & 71.25 {\tiny$\pm$0.27} & \underline{77.13} {\tiny$\pm$0.14} & 77.71 {\tiny$\pm$0.13} \\
    \textbf{IF-CCS}~\cite{DBLP:conf/iclr/ccs} & 52.01 {\tiny$\pm$0.95} & 63.70 {\tiny$\pm$1.26} & 64.31 {\tiny$\pm$0.78} & \underline{74.42} {\tiny$\pm$0.41} & 76.08 {\tiny$\pm$0.22} & 76.77 {\tiny$\pm$0.45} \\
    \textbf{IF-Beta} (Ours) & \textbf{55.78} {\tiny$\pm$0.09} & \textbf{65.95} {\tiny$\pm$0.25} & \textbf{72.32} {\tiny$\pm$0.46} & \textbf{76.73} {\tiny$\pm$0.10} & \textbf{77.51} {\tiny$\pm$0.17} & \textbf{78.27} {\tiny$\pm$0.25} \\

    \bottomrule
\end{tabular}
}
\label{table:cifar100_res18}
\end{table*}

\section{Further Analysis on Sampling Strategies}
\label{app:analysis_sampling}

In this section, we provide a detailed analysis of two representative heuristic sampling methods, CCS~\cite{DBLP:conf/iclr/ccs} and BWS~\cite{bws}, and examine their behavior under KD settings.

\subsection{Threshold Shift in CCS}

CCS first filters out the most difficult samples using a hard cutoff ratio and then applies stratified sampling to maintain diversity within the remaining pool.
However, determining this cutoff ratio is non-trivial, as it must be empirically tuned after full training, which contradicts the goal of efficiency.
A practical workaround is to reuse the hard cutoff ratio suggested in prior studies~\cite{DBLP:conf/iclr/ccs}.
Nevertheless, as shown in \cref{fig:ccs_wo_kd,fig:ccs_w_kd}, we observe that the optimal cutoff ratio shifts notably between KD and non-KD scenarios, even when using the same dataset.
These observations suggest that the hyperparameters provided in in the original CCS paper~\cite{DBLP:conf/iclr/ccs} cannot be directly transferred across distillation settings, thereby limiting the practicality of CCS in efficient KD.

\subsection{Suboptimal Region in BWS}

BWS employs a sliding-window selection mechanism guided by sample difficulty, identifying the optimal window position by shifting a fixed-size window across the sorted samples.
However, this heuristic fixed-size window design ultimately leads to a suboptimal sampling distribution.
To investigate this, we first locate the optimal BWS window and then progressively replace a portion of the samples within it with randomly selected samples from outside the window.
The performance under different replacement ratios is shown in \cref{fig:window_is_not_good}.
As observed, introducing a moderate level of random replacement yields better results than the original fixed window (\ie, 0\% replacement), indicating that the fixed window design in BWS is indeed suboptimal.
These observations underscore the necessity of a more flexible and learnable sampling strategy.

\section{Related Work}
\label{app: related works}

\subsection{Knowledge Distillation}

Knowledge Distillation (KD) aims to transfer the knowledge of a strong teacher network into a lighter student network, typically to improve the student’s performance without modifying its architecture or training protocol~\cite{DBLP:journals/corr15/KnowledgeDistilling}. 
The transferred “knowledge” itself may take different forms: 
(1) Response-based KD, which uses teacher logits or softened probability distributions~\cite{DBLP:journals/corr15/KnowledgeDistilling}; 
(2) Feature-based KD, which matches intermediate activations from hidden layers~\cite{DBLP:journals/corr/fitnets, DBLP:conf/iclr/attentiontransfer};
(3) Relation-based KD, which aligns relational geometry between samples or layers (\eg, pairwise distances or angles)~\cite{DBLP:conf/cvpr/RKD, DBLP:conf/cvpr/AhnHDLD19}; or
(3) Hierarchy-based KD, which distills knowledge at multiple levels or blocks (\eg, transformer layer-wise compression)~\cite{DBLP:conf/acl/mobilebert, DBLP:conf/emnlp/TinyBert}.
Recent works also design different robust KD objectives to sharpen or smooth the student’s logits, such as temperature annealing, margin-based KL variants~\cite{DBLP:conf/cvpr/SunR00C24, iclr25medium}, feature-to-logit consistency terms, or confidence sharpening losses.
In this work, however, we deliberately adopt the most classical and widely used formulation: logit-based distillation using KL divergence.

A complementary research direction seeks to reduce the computational overhead of KD.
Self-knowledge distillation attempts to replace the teacher with latent knowledge hidden in the student itself~\cite{DBLP:conf/eccv/LiangLBZTLF22, DBLP:conf/iccv/YangZLZY023}.
Although this removes the need for a large teacher (thus saving training cost), its accuracy gains remain limited, due to the intrinsic quality gap between the student and a strong teacher.
Another line of acceleration reduces the number of teacher forward passes~\cite{DBLP:conf/eccv/ShenX22, DBLP:conf/cvpr/BeyerZRMA022}.
A common strategy is to pre-compute and store teacher-generated knowledge.
However, fixed pre-computed knowledge may become stale, since the student sees fresh augmented samples every epoch, whereas cached logits are tied to old un-augmented inputs, leading to degradation as empirically shown in~\cite{DBLP:conf/cvpr/BeyerZRMA022}.

In summary, existing KD research mainly focuses on what to distill (logits vs features) or how to speed up distillation (self-knowledge distillation or caching-based).
Our work takes an orthogonal perspective: we investigate which training samples are most essential for KD.
Specifically, we study how logit-based KD behaves when the student is trained not on the full dataset, but on a carefully selected subset obtained via efficient data pruning.

\subsection{Data Pruning}

Data pruning has been extensively studied since the era of classical machine learning~\cite{DBLP:conf/stoc/Har-PeledM04}, which aims to subsets of training samples that maximize informativeness or diversity and minimize redundancy~\cite{2025coresetSurvey}.
However, classical data pruning methods~\cite{DBLP:conf/icml/CampbellB18, DBLP:journals/talg/Clarkson10} designed for conventional machine learning tasks exhibit significant limitations when applied to deep learning, primarily due to their high computational complexity and fixed data representations.
With the rapid advancement of deep learning, recent research has shifted toward methodologies tailored to the characteristics of deep neural networks. 
This trend has become even more pronounced in the era of large-scale foundation models~\cite{DBLP:conf/icml/24Less,DBLP:conf/aaai/BioCoreset}, where data pruning is increasingly important for reducing training cost and improving data efficiency. 
Existing methods can be broadly categorized into score-based approaches and optimization-based approaches.
While optimization-based formulations~\cite{icml22BilevelCoreset, DBLP:conf/iclr/datasetpruning} can yield theoretically grounded subsets, they typically involve heavy computation and large memory overhead, making them impractical for large-scale settings such as ImageNet or effient KD-based workflows. 
Therefore, the most recent advances emphasize score-based approaches.

Score-based methods assign an importance score to each training example and retain samples with high utility. The majority of these techniques rely on training dynamics, such as per-iteration predictions, losses, or gradients obtained during model training. 
Within this paradigm, two major families have emerged:
(1) Difficulty-based score~\cite{el2n, toneva2018forgetting, pleiss2020aum}, which quantify how difficult each example is to learn. Early approaches select the most difficult samples, but recent studies~\cite{DBLP:conf/iclr/ccs, bws} show that this strategy deteriorates under high pruning ratios, even inferior to random pruning, because extremely hard samples often harm generalization. This motivates methods such as CCS~\cite{DBLP:conf/iclr/ccs} and BWS~\cite{bws}, which incorporate improved sampling strategies to overcome this bottleneck.
(2) Uncertainty-based score~\cite{DynamicUncertainty, DUAL}, which estimate sample informativeness by measuring predictive uncertainty throughout training. Similar to difficulty-based approaches, they also require accessing a substantial portion of the training trajectory, and methods such as DUAL~\cite{DUAL} attempt to alleviate this dependence by leveraging early-epoch signals.

However, despite their effectiveness, these methods remain incompatible with efficient KD, because they fundamentally rely on retraining the model to extract training dynamics. 
In practical KD scenarios, we always only have access to a pretrained teacher~\cite{DBLP:journals/corr15/KnowledgeDistilling}, and obtaining full training trajectories would require retraining the teacher from scratch, which is an expensive and impractical requirement. 
This gap highlights the need for pruning strategies that do not depend on costly training dynamics while still being efficient for KD.
Our work moves toward this goal by introducing influence functions~\cite{koh2017understanding} as a promising post-hoc alternative for estimating sample importance without requiring any teacher retraining.

\subsection{Influence Function}

The concept of influence functions dates back to classical statistics~\cite{cook1980characterizations} and was later introduced into deep neural networks (DNNs) by Koh and Liang~\cite{koh2017understanding}.
However, the method has long been criticized for two major limitations: fragility and computational inefficiency.
However, subsequent studies have identified challenges in applying IF to DNNs.
The fragility issue~\cite{iclr21IFisFragile,nips22IFQA} arises because IFs often fail to align with leave-one-out retraining in nonlinear networks, partly due to local non-convexity and the instability of second-order approximations.
On the other hand, the efficiency issue stems from the need to approximate the inverse Hessian, which is computationally expensive and memory-intensive.
To address this, Koh and Liang~\cite{koh2017understanding} originally proposed the LiSSA~\cite{jmlr17lissa} algorithm for efficient Hessian estimation, yet it remained both slow and inaccurate for large-scale models.
TracIn~\cite{nips20tracIn} sidesteps the Hessian inversion by tracking loss changes across multiple training checkpoints, improving scalability but at the cost of large storage requirements and limited applicability in knowledge distillation (KD) settings.
Other acceleration variants, such as FastIF~\cite{emnlp21fastif}, ArnoldiIF~\cite{aaai22arnoldiIF}, and DataInf~\cite{iclr24datainf}, attempt to approximate the influence more efficiently, yet none fundamentally resolve the performance gap between estimated and true influences.
Recently, Ye et al.~\cite{ye2025robust} bridged this gap by introducing a robust formulation of influence functions that leverages the condition of flat validation minima.
This property allows accurate influence estimation without relying on precise Hessian inversion—using only the diagonal Fisher approximation—thus making IF both stable and computationally efficient.

\section{Proofs}
\label{app:ori_PGE}

Here, we provide the proof of \cref{thm:unbiased-pg}.
\begin{proof}
    By the definition of $\Phi(\phi)$, we have
{
\footnotesize
\begin{align}
    \nabla_{\phi}\Phi(\phi) 
    & = \nabla_{\phi} \mathbb{E}_{\bm m \sim \pi_{\phi, r}} \widehat{\mathcal{L}}\left(\theta^{*}( \bs{m})\right)
    \\
    & = \nabla_{\phi} \sum_{\bm m \sim p(\bm m \mid \phi, r) } \widehat{\mathcal{L}}\left(\theta^{*}( \bs{m})\right) \, p(\bm m | \phi, r) \\
    & = \sum_{\bm m \sim p(\bm m \mid \phi, r) } \widehat{\mathcal{L}}\left(\theta^{*}( \bs{m})\right) \frac{\nabla_{\phi} p(\bs{m}|\phi, r)}{p(\bs{m}|\phi, r)} \, p(\bs{m}|\phi, r)  \\
    & = \sum_{\bm m \sim p(\bm m \mid \phi, r) } \widehat{\mathcal{L}}\left(\theta^{*}(\bs{m})\right) \nabla_{\phi}\ln  p(\bs{m}|\phi, r) \,p(\bs{m}|\phi, r) \\
    & = \mathbb{E}_{p(\bs{m}|\phi, r)} \widehat{\mathcal{L}}\left(\theta^{*}(\bs{m})\right) \nabla_{\phi}\ln  p(\bs{m}|\phi, r).
    \label{eqn:ori-PGE}
\end{align}
}
which completes the proof.
\end{proof}

\section{Derivations}
\label{app:derivation}

In this section, we present the full derivation of Eq.~\ref{eq: IF-FVM}. 
Following the classical influence-function framework of Koh and Liang~\cite{koh2017understanding}, our goal is to approximate the excess loss 
$\ell(z_{\text{val}}, \theta^\star_{z_{\text{tr}}}) - \ell(z_{\text{val}}, \theta^\star)$ 
for a validation example $z_{\text{val}}$, where $\theta^\star_{z_{\text{tr}}}$ denotes the minimizer of the perturbed empirical risk:
\begin{equation}
\theta^\star_{z_\text{tr}} := \arg \min_\theta \frac{1}{N} \sum_{n=1}^N \ell(z_n, \theta) + \epsilon \ell(z_\text{tr}, \theta),
\label{eqn:theta-ztr}
\end{equation}
with $z_n \in D_{\text{tr}}$ and a small upweighting $\epsilon > 0$.
More specifically, the IF can be understood as a two-step approximation: (i) the parameter change and (ii) the loss change.

\vspace{0.05in}
\noindent\textbf{Step 1: Parameter Change.}
Instead of directly working with $\theta^\star_{z_{\text{tr}}}$ as in Koh and Liang~\cite{koh2017understanding}, we consider the perturbed SAM objective and define
\begin{equation}
    \tilde{\theta}_T^{z_\text{tr}}
    := \arg \min_\theta \max_{\|\Delta\| \le \gamma} \hat{R}_{\text{val}}(\theta_T + \Delta) + \epsilon \ell(z_\text{tr}, \theta_T + \Delta).
    \label{eqn: parameter change}
\end{equation}
Here, $\tilde{\theta}_T^{z_\text{tr}}$ denotes the perturbed parameter resulting from incorporating the training sample $z_\text{tr}$ into the optimization objective, \ie, Eq.~\ref{eq: SAM}, with a small perturbation $\epsilon$.
Therefore, we next derive the parameter change $\tilde{\theta}_T^{z_\text{tr}} - \tilde{\theta}_T$ in the context of loss minimization, based on the analysis of $\theta^\star_{z_\text{tr}} - \theta^\star$ presented in Koh and Liang~\cite{koh2017understanding}.

Recall that $\tilde{\theta}_T$ minimizes the following SAM object:
\begin{equation}
    \hat{R}^\gamma_\text{val} (\theta_T)
    := \max_{\|\Delta\| \le \gamma} \hat{R}_{\text{val}}(\theta_T + \Delta) \\
    = \max_{\|\Delta\| \le \gamma} \frac{1}{M} \sum_{m=1}^M \ell(z_m, \theta_T + \Delta).
\end{equation}

We further assume that $R$ is twice-differentiable and strongly convex in $\theta_T$, \ie,
\begin{equation}
    \tilde{H}_\text{val}
    := \nabla^2_{\theta_T} \hat{R}_\text{val} (\tilde{\theta}_T)
    = \frac{1}{M} \sum_{m=1}^M \nabla^2_{\tilde{\theta}_T} \ell(z_m, \tilde{\theta}_T),
\end{equation}
is positive definite and hence invertible.

Since $\tilde{\theta}_T^{z_\text{tr}}$ is the minimizer of Eq.~\ref{eqn: parameter change}, assuming $\gamma \to 0$, let us examine its first-order optimality conditions:
\begin{equation}
    0 = \nabla_{\tilde{\theta}_T} R_\text{val}(\tilde{\theta}_T^{z_\text{tr}}) + \epsilon \nabla_{\tilde{\theta}_T} \ell(z, \tilde{\theta}_T^{z_\text{tr}})
\end{equation}

Considering that $\tilde{\theta}_T^{z_\text{tr}} \to \tilde{\theta}_T$ as $\epsilon \to 0$, we perform a Taylor expansion of the right-hand side:
{
\footnotesize
\begin{equation}
    0
    \approx \left[ \nabla_{\tilde{\theta}_T} R_\text{val}(\tilde{\theta}_T) + \epsilon \nabla_{\tilde{\theta}_T} \ell (z_\text{tr}, \tilde{\theta}_T) \right]
    + \left[ \nabla^2_{\tilde{\theta}_T} R_\text{val}(\tilde{\theta}_T) + \epsilon \nabla^2_{\tilde{\theta}_T} \ell(z_\text{tr}, \tilde{\theta}_T) \right] \left( \tilde{\theta}_T^{z_\text{tr}} - \tilde{\theta}_T \right),
\end{equation}
}
where we have dropped $o(\| \tilde{\theta}_T^{z_\text{tr}} - \tilde{\theta}_T \|)$ terms.

Solving for $\tilde{\theta}_T^{z_\text{tr}} - \tilde{\theta}_T$, we get:
{
\footnotesize
\begin{equation}
    \tilde{\theta}_T^{z_\text{tr}} - \tilde{\theta}_T
    \approx - \left[ \nabla^2_{\tilde{\theta}_T} R_\text{val}(\tilde{\theta}_T) + \epsilon \nabla^2_{\tilde{\theta}_T} \ell(z_\text{tr}, \tilde{\theta}_T) \right]^{-1}
    \left[ \nabla_{\tilde{\theta}_T} R_\text{val}(\tilde{\theta}_T) + \epsilon \nabla_{\tilde{\theta}_T} \ell (z_\text{tr}, \tilde{\theta}_T) \right].
\end{equation}
}

Since $\tilde{\theta}_T$ is the minimizer of $\hat{R}^\gamma_\text{val}(\theta_T)$, we have $\nabla_{\tilde{\theta}_T} R_\text{val}(\tilde{\theta}_T) = 0$.
Dropping $o(\epsilon)$ terms, then
\begin{equation}
    \tilde{\theta}_T^{z_\text{tr}} - \tilde{\theta}_T
    \approx - \epsilon \tilde{H}^{-1}_\text{val} \tilde{g}_{z_\text{tr}},
    \label{eq: parameter change}
\end{equation}
where $\tilde{g}_{z_\text{tr}} = \nabla_{\tilde{\theta}} \ell (z_\text{tr}, \tilde{\theta})$.

\vspace{0.05in}
\noindent\textbf{Step 2: Loss Change.}
Consider the loss change on a validation sample $z_\text{val}$,
and we denote the parameter change $\Delta \theta_T = \tilde{\theta}_T^{z_\text{tr}} - \tilde{\theta}_T$. 
We propose to estimate the loss change with respect to the parameter changes $\Delta \theta_T$ via second-order approximation:
\begin{equation}
    \ell(z_\text{val}, \tilde{\theta}_T^{z_\text{tr}}) - \ell(z_\text{val}, \tilde{\theta}_T)
    \approx
    \tilde{g}_{z_\text{val}}^\top \Delta \theta_T +
    \frac{1}{2} \Delta \theta_T^\top \nabla^2_{\tilde{\theta}_T} \ell (z_{\text{val}}, \tilde{\theta}_T) \Delta \theta_T,
\end{equation}
where $\tilde{g}_{z_\text{val}} = \nabla_{\theta_T} \ell(z_\text{val}, \tilde{\theta}_T)$.

Next, combining with Eq.~\ref{eq: parameter change}, we have
{
\footnotesize
\begin{align}
    \ell(z_\text{val}, \tilde{\theta}_T^{z_\text{tr}}) - \ell(z_\text{val}, \tilde{\theta}_T)
    & \approx
    \tilde{g}_{z_\text{val}}^\top \left( - \epsilon \tilde{H}^{-1}_\text{val} \tilde{g}_{z_\text{tr}} \right) \\
    & \quad + \frac{1}{2} \left( - \epsilon \tilde{H}^{-1}_\text{val} \tilde{g}_{z_\text{tr}} \right)^\top \nabla^2_{\tilde{\theta}_T} \ell (z_{\text{val}}, \tilde{\theta}_T) \left( - \epsilon \tilde{H}^{-1}_\text{val} \tilde{g}_{z_\text{tr}} \right) \\
    & =
    - \epsilon \tilde{g}_{z_\text{val}}^\top \tilde{H}^{-1}_\text{val} \tilde{g}_{z_\text{tr}}+
    \frac{1}{2} \epsilon^2 \tilde{g}_{z_\text{tr}}^\top \tilde{H}^{-1}_\text{val} \nabla^2_{\tilde{\theta}_T} \ell (z_{\text{val}}, \tilde{\theta}_T) \tilde{H}^{-1}_\text{val} \tilde{g}_{z_\text{tr}}.
\end{align}
}

Since we want to measure the loss change with respect to removing the training sample $z_\text{tr}$, we define the Influence Function as follows:
\begin{equation}
    \label{eq: influence on single test sample}
    \mathcal{I}(z_\text{tr}, z_\text{val})
    :=
    \tilde{g}_{z_\text{val}} \tilde{H}^{-1}_\text{val} \tilde{g}_{z_\text{tr}}
    + \frac{1}{2} \epsilon \tilde{g}_{z_\text{tr}}^\top \tilde{H}^{-1}_\text{val} \nabla^2_{\tilde{\theta}_T} \ell (z_{\text{val}}, \tilde{\theta}_T) \tilde{H}^{-1}_\text{val} \tilde{g}_{z_\text{tr}},
\end{equation}
with $\epsilon > 0$.

Then, we consider the influence on the validation set $D_\text{val}$, which can be defined as the sum of the influence on each validation sample:
\begin{equation}
    \mathcal{I}(z_\text{tr}, D_\text{val})
    := \sum_{m=1}^{M} \mathcal{I}(z_\text{tr}, z_m).
\end{equation}

Incorporating with \cref{eq: influence on single test sample}, we get
\begin{equation}
    \begin{aligned}
        \mathcal{I}(z_\text{tr}, D_\text{val})
        & := \sum_{m=1}^M
        \tilde{g}_{z_m} \tilde{H}^{-1}_\text{val} \tilde{g}_{z_\text{tr}}
        + \frac{1}{2} \epsilon \tilde{g}_{z_\text{tr}}^\top \tilde{H}^{-1}_\text{val} \nabla^2_{\tilde{\theta}_T} \ell (z_{\text{val}}, \tilde{\theta}_T) \tilde{H}^{-1}_\text{val} \tilde{g}_{z_\text{tr}} \\
        & = \tilde{g}_\text{val} \tilde{H}^{-1}_\text{val} \tilde{g}_{z_\text{tr}}
        + \frac{1}{2} \epsilon \tilde{g}_{z_\text{tr}}^\top \tilde{H}^{-1}_\text{val} \tilde{g}_{z_\text{tr}}.
    \end{aligned}
\end{equation}

Since $\tilde{g}_\text{val} \to 0$, we define the Influence Function on the validation set $D_\text{val}$ as follows:
\begin{equation}
    \mathcal{I}(z_\text{tr}, D_\text{val})
    := \tilde{g}_{z_\text{tr}}^\top \tilde{H}^{-1}_\text{val} \tilde{g}_{z_\text{tr}},
\end{equation}
where $\frac{1}{2} \epsilon > 0$ is dropped.

\end{document}